%% file: main.tex
\documentclass[twoside,11pt]{article}

\usepackage{jmlr2e}

\usepackage[usenames,dvipsnames]{xcolor}
\usepackage{tikz,pgfplots}
\pgfplotsset{compat=newest}
\usetikzlibrary{arrows,automata}
\usetikzlibrary{arrows.meta,calc,decorations.markings,math,arrows.meta}

\usepackage{mathtools}
\usepackage{amsfonts,amssymb}
\usepackage{bm}
\usepackage{bbm}
\usepackage{cases}

\usepackage{enumitem}
\usepackage{tabularx}
    \newcolumntype{L}{>{\raggedright\arraybackslash}X}
    \newcolumntype{C}{>{\centering\arraybackslash}X}

\usepackage{xspace}
\usepackage{wrapfig}

\usepackage{graphicx}
\usepackage{subcaption}
\usepackage{enumitem}
\usepackage{xfrac}
\usepackage{fontawesome}

\usepackage[colorinlistoftodos, textwidth=20mm]{todonotes}
\definecolor{citrine}{rgb}{0.89, 0.82, 0.04}
\definecolor{blued}{RGB}{70,197,221}

\input{symbolsdef}

\jmlrheading{1}{2000}{1-48}{4/00}{10/00}{meila00a}{Qian, Fruit, Pirotta and Lazaric}

\ShortHeadings{Exploration Bonus for Regret Minimization}{Qian, Fruit, Pirotta and Lazaric}
\firstpageno{1}

\begin{document}

\title{Exploration Bonus for Regret Minimization in Undiscounted Discrete and Continuous Markov Decision Processes}
\author{%
          Jian Qian
          \email{jian.qian@ens.fr}\\
          \addr Sequel Team - Inria Lille
          \AND
          Ronan Fruit
          \email{ronan.fruit@inria.fr}\\
          \addr Sequel Team - Inria Lille
          \AND
          Matteo Pirotta\thanks{Work done while at INRIA Lille (Sequel Team).}
          \email{pirotta@fb.com}\\
          \addr Facebook AI Research 
          \AND
          Alessandro Lazaric 
          \email{lazaric@fb.com}\\
          \addr Facebook AI Research
}

\editor{}

\maketitle

\begin{abstract}
        We introduce and analyse two algorithms for exploration-exploitation in discrete and continuous Markov Decision Processes (MDPs) based on \emph{exploration bonuses}.
        \bonusscal is a variant of \scal~\citep{fruit2018constrained} that performs efficient exploration-exploitation in any \emph{unknown weakly-communicating} MDP for which an upper bound $c$ on the span of the \emph{optimal bias function} is known.
        For an MDP with $S$ states, $A$ actions and $\nextstates \leq S$ possible next states, we prove that \bonusscal achieves the same theoretical guarantees as \scal (\ie a high probability regret bound of $\wt{O}(c\sqrt{\nextstates SAT})$), with a much smaller computational complexity.
        Similarly, \scalcont exploits an exploration bonus to achieve sublinear regret in any undiscounted MDP with continuous state space. 
        We show that \scalcont achieves the same regret bound as \uccrl~\citep{DBLP:journals/corr/abs-1302-2550} while being the first implementable algorithm with regret guarantees in this setting.
        While optimistic algorithms such as \ucrl, \scal or \uccrl maintain a high-confidence set of plausible MDPs around the true unknown MDP, \bonusscal and \scalcont leverage on an exploration bonus to directly plan on the empirically estimated MDP, thus being more computationally efficient.
\end{abstract}

\input{sec_introduction}

\input{sec_preliminaries}

\input{sec_bonus_finite}

\input{sec_continuous}

\input{sec_conclusion}

\clearpage
\bibliography{span}

\clearpage
\appendix
\textbf{Structure of the appendix.}
We first present the proofs for the continuous case (App.~\ref{app:continuous}) and then for the discrete case (App.~\ref{app:finite}) since the latter can be viewed as a special case of the former. That way, we only need to highlight the main differences in the discrete case.
Because of the continuous nature of the state space and the aggregation of states in the continuous case, extra care needs to be taken while using concentration inequalities compared to standard regret proofs in RL (see App.~\ref{app:bound_diff_tilde_bar} for more details).
In App.~\ref{app:prob.res} we recall/prove all the necessary results from probability theory that we use in the regret proofs.
\input{app_regret_continuous}
\input{app_proof_regret_cont}

\input{app_regret_finite}

\input{app_probability}

\end{document}

%% file: symbolsdef.tex
\newcommand{\A}{\mathcal A}

\newcommand{\calS}{\mathcal S}

\newcommand{\G}{\mathcal G}

\newcommand{\F}{\mathcal F}
\newcommand{\Prob}{\mathbbm P}
\renewcommand{\Re}{\mathbb R}
\newcommand{\Na}{\mathbb N}
\newcommand{\E}{\mathbb E}
\newcommand{\nextstates}{\Gamma}

\newcommand{\one}[1]{\mathbbm{1} \left(#1 \right)}
\newcommand{\sigalg}{$\sigma$-algebra\@\xspace}
\newcommand{\sigalgs}{$\sigma$-algebras\@\xspace}
\newcommand{\as}{a.s.\@\xspace}
\newcommand{\rv}{r.v.\@\xspace}
\newcommand{\rvrv}{real-valued r.v.\@\xspace}
\newcommand{\smallplus}{{\scalebox{.5}{+}}}

\newcommand{\SR}{\Pi^\text{SR}}
\newcommand{\SD}{\Pi^\text{SD}}
\newcommand{\PiC}{\Pi_c}

\newcommand{\Varcond}[2]{\mathbb{V}\left({#1} \big| {#2}\right)}
\newcommand{\SP}[1]{sp\left\{{#1}\right\}}
\newcommand{\proj}[1]{\Gamma_c{#1}}

\DeclareMathOperator*{\argmax}{\arg\,\max}

\newcommand{\transp}{\mathsf{T}}
\newcommand{\opT}[1]{T_c#1}

\newcommand{\lnn}[1]{\ln\left({#1}\right)}
\newcommand{\expp}[1]{\exp\left({#1}\right)}

\newcommand{\scopt}{{\small\textsc{ScOpt}}\xspace}
\newcommand{\ucrl}{{\small\textsc{UCRL}}\xspace}
\newcommand{\uccrl}{{\small\textsc{UCCRL}}\xspace}

\newcommand{\bonusscal}{{\small\textsc{SCAL}\textsuperscript{+}}\xspace}

\newcommand{\optpsrl}{{\small\textsc{Opt-PSRL}}\xspace}

\newcommand{\scal}{{\small\textsc{SCAL}}\xspace}
\newcommand{\scalcont}{{\small\textsc{C-SCAL}\textsuperscript{+}}\xspace}
\newcommand{\tucrl}{{\small\textsc{TUCRL}}\xspace}

\newcommand{\rmaxbound}{r_{\max}}

\newcommand{\wt}[1]{\widetilde{#1}}
\newcommand{\wh}[1]{\widehat{#1}}
\newcommand{\wo}[1]{\overline{#1}}
\newcommand{\wb}[1]{\overline{#1}}

\usepackage{accents}
\DeclareMathAccent{\wtilde}{\mathord}{largesymbols}{"65}

\newcommand{\pluseq}{\mathrel{+}=}

\DeclareRobustCommand{\eg}{e.g.,\@\xspace}
\DeclareRobustCommand{\ie}{i.e.,\@\xspace}

\DeclareRobustCommand{\wrt}{w.r.t.\@\xspace}

\DeclareRobustCommand{\st}{s.t.\@\xspace}

\newtheorem{assumption}[theorem]{Assumption}

\newlength{\minipagewidth}
\newlength{\minipagewidthx}
\newcommand{\bookboxx}[1]{\small
\par\medskip\noindent
\framebox[0.99\textwidth]{
\begin{minipage}{0.97\dimexpr\textwidth-\parindent\relax} {#1} \end{minipage} } \par\medskip }

%% file: sec_introduction.tex
\section{Introduction}
While learning in an unknown environment, a reinforcement learning (RL) agent must trade off the \textit{exploration} needed to collect information about the dynamics and reward, and the \textit{exploitation} of the experience gathered so far to gain as much reward as possible. 
The performance of an online learning agent is usually measured in terms of cumulative regret which compares the rewards accumulated by the agent with the rewards accumulated by an optimal agent.
A popular strategy to deal with the exploration-exploitation dilemma (\ie minimize regret) is to follow the \emph{optimism in the face of uncertainty} (OFU) principle.

Optimistic approaches have been widely studied in the context of stochastic multi-armed bandit (MAB) problems.
In this setting, OFU-based algorithms maintain optimistic estimates of the expected reward of each action $a$ (\ie arm), and play the action with highest optimistic estimate \citep[see \eg][]{bubeck2012regret,bandittorcsaba}.
These optimistic estimates are usually obtained by adding a high probability \emph{confidence bound} $b(a)$ to the empirical average reward $\wh{r}(a)$ \ie $\wh{r}(a) + b(a)$. The confidence bound plays the role of an \emph{exploration bonus}: the higher $b(a)$, the more likely $a$ will be explored.
As an example, based on Hoeffding's inequality, the Upper-Confidence Bound (UCB) algorithm uses $b(a) = \wt{\Theta}\big(\rmaxbound/\sqrt{N(a)}\big)$ where $N(a)$ is the total number of times action $a$ has been played before and all rewards are assumed to lie between $0$ and $\rmaxbound$ with probability $1$. UCB can be shown to achieve nearly-optimal regret guarantees.

\citet{strehl2008analysis} later generalized the idea of enforcing exploration by using a bonus on the reward to the RL framework. 
They analysed the \emph{infinite-horizon $\gamma$-discounted setting} and introduced the Model Based Interval Estimation with Exploration Bonus (MBIE-EB) algorithm. MBIE-EB plays the optimal policy of the empirically estimated MDP where for each state-action pair $(s,a)$, a bonus ${b}(s,a)$ is added to the empirical average reward $\wh{r}(s,a)$ \ie the immediate reward associated to $(s,a)$ is $\wh{r}(s,a) + b(s,a)$. Unlike in MAB where the optimal arm is the one with maximal immediate reward, the goal of RL is to find a policy maximizing the cumulative reward \ie the $Q$-function. Therefore, the bonus needs to account for the uncertainty in both the rewards and transition probabilities and so $b(s,a) = \wt{\Theta}\left(\frac{\rmaxbound}{1-\gamma} \sqrt{\frac{1}{N(s,a)}}\right)$ where $\frac{\rmaxbound}{1-\gamma}$ is the range of the $Q$-function. \citet{strehl2008analysis} also derived PAC guarantees on the sample complexity of MBIE-EB. 
More recently, \emph{count-based methods} \citep[\eg][]{Bellemare:2016aa,tang2017exploration,ostrovski2017count-based,MartinSEH17} tried to combine the idea of MBIE-EB with Deep RL (DRL) techniques to achieve a good exploration-exploitation trade off in high dimensional problems. The exploration bonus usually used has a similar form $\wt{\Theta} \left(\frac{\beta}{\sqrt{N}}\right)$ where $\beta$ is now an hyper-parameter tuned for the specific task at hand, and the visit count $N$ is approximated using discretization (\eg hashing) or density estimation methods.

Exploration bonuses have also been successfully applied to \emph{finite-horizon problems} \citep{pmlr-v70-azar17a,kakade2018variance,qlearning2018}. In this setting, the planning horizon $H$ is known to the learning agent and the range of the $Q$-function is $\rmaxbound H$. A natural choice for the bonus is then $b(s,a) = \wt{\Theta}\big(\rmaxbound H/\sqrt{N(s,a)} \big)$. UCBVI\_1 introduced by \citet{pmlr-v70-azar17a} uses such a bonus and achieves near-optimal regret guarantees $\wt{O}\big(H\sqrt{SAT}\big)$. Extensions of UCBVI\_1 exploiting the variance instead of the range of the $Q$-function achieve a better regret bound $\wt{O}\big(\sqrt{HSAT}\big)$ \citep{pmlr-v70-azar17a,kakade2018variance,qlearning2018}.

Both the finite horizon setting and infinite horizon discounted setting assume that there exists an \emph{intrinsic horizon} (respectively $H$ and $\frac{1}{1-\gamma}$) known to the learning agent. Unfortunately, in many common RL problems it is not clear how to define $H$ or $\frac{1}{1-\gamma}$ and it is often desirable to set them as big as possible (e.g., in episodic problem, the time to the goal is not known in advance and random in general). As $H$ tends to infinity the regret (of UCBVI\_1, etc.) will become linear while as $\gamma$ tends to to $1$ the sample complexity (of MBIE-EB, etc.) tends to infinity (not to mention the numerical instabilities that may arise). In this paper we focus on the much more natural infinite horizon undiscounted setting \citep[Chap. 8]{puterman1994markov} which generalizes the two previous settings to the case where $H\to +\infty$ and $\gamma \to 1$ respectively. Several algorithms implementing the OFU principle in the infinite horizon undiscounted case have been proposed in the literature \citep[\eg][]{Jaksch10,DBLP:journals/corr/abs-1302-2550,fruit2017optionsnoprior,fruit2018constrained,Talebi2018variance}, but none of these approaches exploits the idea of an exploration bonus. Instead, they all construct an \emph{extended} MDP\footnote{The extended MDP is sometimes called \emph{bounded-parameter} MDP} with continuous action space, which can be interpreted as the concatenation of all possible MDPs compatible with some high probability confidence bounds on the transition model, among which is the true MDP. The policy executed by the algorithm is the optimal policy of the extended MDP. \ucrl \citep{Jaksch10} achieves a regret of order\footnote{The original bound of \citet{Jaksch10} has $\sqrt{S}$ instead of $\sqrt{\Gamma}$ but $\sqrt{\Gamma}$ can be easily achieved by replacing Hoeffding inequality by empirical Bernstein's inequality for transition probabilities.} $\wt{O}\big(\rmaxbound D \sqrt{\Gamma S AT}\big)$ after $T$ time steps where $D$, $\Gamma$, $S$ and $A$ are respectively the diameter of the true MDP, the maximum number of reachable next states from any state, the number of states and the number of actions. \citep{fruit2018constrained} showed an improved bound for \scal $\wt{O}\big( \min{\{\rmaxbound D,c \}} \sqrt{\Gamma S AT}\big)$ when a known upper bound on the optimal bias span $c \geq \SP{h^*}$ is known to the learning agent. Although such algorithms can be efficiently implemented in the tabular case, it is difficult to extend them to more scalable approaches like DRL. In contrast, as already mentioned, the exploration bonus approach is simpler to adapt to large scale problems and inspired count based methods in DRL. 

In this paper we introduce and analyse \bonusscal, the first algorithm that relies on an exploration bonus to efficiently balance exploration and exploitation in the infinite-horizon undiscounted setting.
All the exploration bonuses that were previously introduced in the RL literature explicitly depend on $\gamma$ or $H$ which are known to the learning agent.
In the infinite-horizon undiscounted case, there is no predefined parameter informing the agent about the range of the $Q$-function. This makes the design of an exploration bonus very challenging. 
To overcome this limitation, we make the same assumption as \citet{Bartlett2009regal, fruit2018constrained} \ie we assume that the agent knows an upper-bound $c$ on the span (\ie range) of the optimal bias (\ie value function).
The exploration bonus used by \bonusscal is thus $b(s,a) = \wt{\Theta}\big(\max\{c,\rmaxbound\} /\sqrt{N(s,a)}\big)$. In comparison, state-of-the-art algorithms in the infinite horizon undiscounted setting like \ucrl or \scal can, to a certain extent, be interpreted as virtually using an exploration bonus of order $\wt{\Theta}\big(\rmaxbound D \sqrt{\Gamma/N(s,a)}\big)$ and $\wt{\Theta}\big(\max\{c,\rmaxbound\} \sqrt{\Gamma/N(s,a)}\big)$ respectively. This is bigger by a multiplicative factor $\sqrt{\Gamma}$. As a result, to the best of our knowledge, \bonusscal achieves a \emph{``tighter''} optimism than any other existing algorithm in the infinite horizon undiscounted setting and is therefore less prone to \emph{over-exploration}. 

To further illustrate the generality of the exploration bonus approach, we also present \scalcont, an extension of \bonusscal to continuous state space --but finite action space-- MDPs.
As in~\citep{DBLP:journals/corr/abs-1302-2550,pmlr-v37-lakshmanan15}, we require the reward and transition functions to be H\"{o}lder continuous with parameters $L$ and $\alpha$.
\scalcont is also the first implementable algorithm in continuous problem with theoretical guarantees (existing algorithms with theoretical guarantees such as \uccrl \citep{DBLP:journals/corr/abs-1302-2550} cannot be implemented).
\scalcont combines the idea of \bonusscal with state aggregation. Compared to \bonusscal, the exploration bonus contains an additional term due to the discretization: for any aggregated state $I$, $b(I,a) = \wt{O}\left(\max\{c,\rmaxbound\} \big(1/\sqrt{N(I,a)} + LS^{-\alpha} \big) \right)$.

The main result of the paper is summarized in Thm.~\ref{thm:summary}:
\begin{theorem}\label{thm:summary}
        For any MDP with $S$ states, $A$ actions and $\Gamma$ next states, the regret of \bonusscal is bounded with  high probability by
        $
                \wt{O}\left( \max\{c,\rmaxbound\}  \sqrt{\Gamma S AT} \right).
        $
        For any ``smooth'' MDP with smoothness parameters $L$ and $\alpha$, $1$-dimensional state space $\calS = [0,1]$ and $A$ actions, the regret of \scalcont is bounded with high probability by
        $
                \wt{O} \left(\max\{c,\rmaxbound\} L \sqrt{A} T^{\sfrac{(\alpha+2)}{(2\alpha+2)}} \right). 
        $
\end{theorem}
The regret bound of \bonusscal (resp. \scalcont) matches the one of \scal (\uccrl). Surprisingly, the tighter optimism introduced by \bonusscal compared to \scal and \ucrl is not reflected in the final regret bound with the current statistical analysis ($\sqrt{\Gamma}$ appears in the bound although despite not being included in the bonus). We isolate and discuss where the term $\sqrt{\Gamma}$ appears in the proof sketch of Sect.~\ref{app:regret.bonusscal}.  While \citet{pmlr-v70-azar17a,kakade2018variance,qlearning2018} managed to remove the $\sqrt{\Gamma}$ term in the finite horizon setting, it remains an open question whether their result can be extended to the infinite horizon case (for example, the two definitions of regret do not match and differ by a linear term) or it is an intrinsic difficulty of the setting.
Finally, \bonusscal and \scalcont are very appealing due to their simplicity and flexibility of implementation since the planning is performed on the empirical MDP (rather than on a much more complex extended MDP).
This change of paradigm results in a more computationally efficient planning compared to \ucrl and \scal, as explained in Sec.~\ref{sec:algo}.

%% file: sec_preliminaries.tex
\section{Preliminaries}\label{sec:preliminaries}

\subsection{Markov Decision Processes}\label{sec:mdp}
We consider a \emph{weakly-communicating}\footnote{
In a weakly-communicating MDP, the set $\calS$ can be decomposed into two subsets: a \emph{communicating} set in which for any pair of states $s,s'$ there exists a policy that has a non-zero probability to reach $s'$ starting from $s$, and a set of states that are \emph{transient} under all policies.}
MDP \citep[Sec. 8.3]{puterman1994markov} $M = ( \calS, \A, p, r )$ with a set of states $\calS$ and a set of actions $\A$.
For sake of clarity, here we consider a finite MDP $M$ but all the stated concepts extend to the case of continuous state space,~\citep[see \eg][]{DBLP:journals/corr/abs-1302-2550}.

Each state-action pair $(s,a)$ is characterized by a reward distribution with mean $r(s,a)$ and support in $[0, \rmaxbound]$ as well as a transition probability distribution $p(\cdot|s,a)$ over next states.
We denote by $S = |\calS|$ and $A = |\A|$ the number of states and action, and by $\nextstates = \max_{s \in \calS, a \in \A} \|p(\cdot|s,a)\|_{0} \leq S$ the maximum support of all transition probabilities $p(\cdot|s,a)$.
A stationary Markov randomized policy $\pi : \calS \rightarrow P(\A)$ maps states to distributions over actions.
The set of stationary randomized (resp. deterministic) policies is denoted by $\SR$ (resp. $\SD$).
Any policy $\pi \in \SR(M)$ has an associated \emph{long-term average reward} (or gain) and a \emph{bias function} defined as
\begin{align*}
        g^\pi(s) := \lim_{T\to +\infty} \mathbb{E}_{\mathbb{Q}}\Bigg[ \frac{1}{T}\sum_{t=1}^T r(s_t,a_t) \Bigg],~~ 
        h^\pi(s) := \underset{T\to +\infty}{C\text{-}\lim}~\mathbb{E}_{\mathbb{Q}}\Bigg[\sum_{t=1}^{T} \big(r(s_t,a_t) - g^\pi(s_t)\big)\Bigg],
\end{align*}
where $\mathbb{Q} := \mathbb{P}\left(\cdot|a_t \sim \pi(s_t); s_0=s; M\right)$ and the bias $h^\pi(s)$
 measures the expected total difference between the reward and the stationary reward in \emph{Cesaro-limit}\footnote{For policies with an aperiodic chain, the standard limit exists.} (denoted $C\text{-}\lim$). Accordingly, the difference of bias values $h^\pi(s)-h^\pi(s')$ quantifies the (dis-)advantage of starting in state $s$ rather than $s'$ and we denote by $\SP{h^\pi} := \max_s h^\pi(s) - \min_{s} h^\pi(s)$ the \emph{span} of the bias function.
In weakly communicating MDPs, any optimal policy $\pi^* \in \argmax_\pi g^\pi(s)$ has \emph{constant} gain, i.e., $g^{\pi^*}(s) = g^*$ for all $s\in\calS$. 
Moreover, there exists a policy $\pi^* \in \argmax_\pi g^\pi(s)$ for which $(g^*,h^*) = (g^{\pi^*},h^{\pi^*})$ satisfy the \emph{optimality equation}
\begin{equation}\label{eq:optimality.equation}
                h^* = L h^* - g^* e, \quad\text{ where }\;\; e = (1,\dots,1)^\intercal.
\end{equation}
where L is the \emph{optimal} Bellman operator: $\forall v \in \Re^S, s\in\calS,$
\begin{equation}\label{eq:optimal.bellman.op}
  \begin{aligned}
          Lv(s) &:=
        \max_{a \in \A} \{r(s,a) + p(\cdot|s,a)^\transp v\}
  \end{aligned}
\end{equation}
Note that $h^*$ is finite, \ie $\SP{h^*} < +\infty$.
Finally,
    $D := \max_{(s,s') \in \calS\times \calS, s\neq s'} \{\tau(s \to s')\}$
denotes the diameter of $M$,
where $\tau(s\to s')$ is the minimal expected number of steps needed to reach $s'$ from $s$ in $M$ (under any policy).

\subsection{Planning under span constraint}\label{sec:scopt}
In this section we introduce and analyse the problem of planning under bias span constraint, \ie by imposing that $\SP{h^\pi} \leq c$, for any policy $\pi$.
This problem is at the core of the proposed algorithms (\bonusscal and \scalcont) for exploration-exploitation.
Formally, we define the optimization problem:
\begin{equation}\label{eq:opt.superior.spanc}
        g^*_c(M) := \sup_{\pi \in \PiC(M)}\{g^\pi\},
\end{equation}
where $M$ is any MDP (with discrete or continuous state space) \st $\PiC(M) := \{\pi \in \SR : \SP{h^\pi} \leq c \wedge \SP{g^\pi} = 0\} \neq \emptyset$.\footnote{\citet[][Lem. 2]{fruit2018constrained} showed that there may not exist a deterministic optimal policy for problem~\ref{eq:opt.superior.spanc}.}
This problem is a slight variation of the bias-span constrained problem considered  by~\citep{Bartlett2009regal,DBLP:journals/corr/abs-1302-2550,pmlr-v37-lakshmanan15}, for which no known-solution is available.
On the other hand, problem~\ref{eq:opt.superior.spanc} has been widely analysed by~\citet{fruit2018constrained}.


Problem~\ref{eq:opt.superior.spanc} can be solved using \scopt~\citep{fruit2018constrained}, a version of (relative) value iteration~\citep{puterman1994markov,bertsekas1995dynamic}, where the optimal Bellman operator is modified to return value functions with span bounded by $c$, and the stopping condition is tailored to return a constrained-greedy policy with near-optimal gain.
Given $v \in \Re^{S}$ and $c \geq 0$, we define the value operator $\opT{}:\Re^{S}\rightarrow \Re^{S}$ as
        \begin{align}\label{eq:opT}
                \opT{v} = \proj{Lv} = \begin{cases}
                        Lv(s) & \forall s \in \overline{\calS}(c,v)\\
                        c + \min_s \{Lv(s)\} & \forall s \in \mathcal{S} \setminus \overline{\calS}(c,v)
                \end{cases}
        \end{align}
        where $\overline{\calS}(c,v) = \left\{s \in \mathcal{S} | Lv(s) \leq \min_s \{Lv(s)\} + c \right\}$ and $\proj{}$ is the span constrain projection operator (see~\citep[][App. D]{fruit2018constrained} for details). 
        In other words, operator $\opT{}$ applies a \emph{span truncation} to the one-step application of $L$, which guarantees that $\SP{\opT{v}} \leq c$. 
Given a vector $v_0\in\Re^S$ and a reference state $\wb s$ \scopt implements relative value iteration where $L$ is replaced by $T_c$:
 $v_{n+1} = \opT{v_{n}} - \opT{v_{n}}(\wb{s})e$.
 We can now state the convergence guarantees of \scopt~\citep[see ][Lem. 8 and Thm. 10]{fruit2018constrained}.
\begin{proposition}\label{prop:recap_scopt}
        Let's assume that I) the optimal Bellman operator $L$ is a $\gamma$-span-contraction; II) all policies are unichain; III) operator $\opT{}$ is globally feasible at any vector $v \in \mathbb{R}^S$ such that $\SP{v} \leq c$ \ie for all $s \in \calS,~\min_{a \in \A} \{r(s,a) + p(\cdot|s,a)^\transp v \} \leq \min_{s'} \{Lv(s')\} + c$.
        Then:
        \begin{enumerate}[label=(\alph*)]
                \item \emph{Optimality equation:} there exists a solution $(g^+,h^+) \in \Re \times \Re^{S}$ to the optimality equation $\opT{h^+} = h^+ + g^+ e$. Moreover, any solution $(g^+,h^+)$ satisfies $g^+ = g^*_c$.
                \item \emph{Convergence:} for any initial vector $v_0 \in \mathbb{R}^S$, \scopt converges to a solution $h^+$ of the optimality equation, and $\lim_{n \to +\infty}\opT{}^{n+1}v_0-\opT{}^n v_0 = g^+e$. 
        \end{enumerate}
\end{proposition}

%
%
 
\subsection{Learning Problem}
Let $M^*$ be the true \emph{unknown} MDP. We consider the learning problem where $\mathcal{S}$, $\mathcal{A}$ and $\rmaxbound$ are \emph{known}, while rewards $r$ and transition probabilities $p$ are \emph{unknown} and need to be estimated on-line. We evaluate the performance of a learning algorithm $\mathfrak{A}$ after $T$ time steps by its cumulative \emph{regret}:
  $\Delta(\mathfrak{A},T) = T g^* - \sum_{t=1}^T r_t(s_t,a_t).$
%
Finally, we assume that the algorithm is provided with the knowledge of a constant $c>0$ such that $\SP{h^*} \leq c$.
This assumption has been widely used in the literature~\citep[see \eg][]{DBLP:journals/mima/Ortner08,DBLP:journals/corr/abs-1302-2550,fruit2018constrained} and, as shown by~\citep{fruit2018truncated}, it is necessary in order to achieve a logarithmic regret bound in weakly-communicating MDPs.

%% file: sec_bonus_finite.tex
\section{\bonusscal: \scal with exploration bonus}
\label{S:exploration.bonus}

In this section, we introduce \bonusscal, the first online RL algorithm --in the infinite horizon undiscounted setting-- that leverages on an exploration bonus to achieve provable good regret guarantees. Similarly to \scal \citep{fruit2018constrained}, \bonusscal takes advantage of the prior knowledge on the optimal bias span $\SP{h^*} \leq c$ through the use of \scopt. In Sec.~\ref{sec:algo} we present the details of \bonusscal and we give an explicit formula for the exploration bonus. 
We then show that all the conditions of Prop.~\ref{prop:recap_scopt} are satisfied for \bonusscal, meaning that \scopt can be used.
Finally, we justify the choice of the bonus by showing that \bonusscal is gain-optimistic (Sec.~\ref{sec:exploration_bonus}) and we conclude this section with the regret guarantees of \bonusscal (Thm.~\ref{thm:regret.bonusscal}) and a sketch of the regret proof.

\subsection{The algorithm}\label{sec:algo}
\bonusscal is a variant of \scal that uses \scopt to (approximately) solve \eqref{eq:opt.superior.spanc} on MDP $\wh{M}_k^+$ at the beginning of each episode $k$ (see Fig.~\ref{fig:ucrl.constrained}).\footnote{The algorithm is reported in its general form, which applies to both finite and continuous MDPs.}
Before defining $\wh{M}_k^+$ we need to introduce some notations and an intermediate MDP $\wh{M}_k$.

Denote by $t_k$ the starting time of episode $k$,  $N_k(s,a,s')$ the number of observations of 3-tuples $(s,a,s')$ before episode $k$ ($k$ excluded) and $N_k(s,a) := \sum_{s'}N_k(s,a,s')$. 
As in UCRL, we define the empirical averages $\wo{r}_k(s,a)$ and $\wo{p}_k(\cdot|s,a)$ by:
\begin{align*}
        \wo{p}_k(s'|s,a) := \frac{N_k(s,a,s')}{N_k(s,a)}  ~~\text{ and }~~  \wo{r}_k(s,a) := \frac{1}{N_k(s,a)} \sum_{t=1}^{t_k -1} r_t(s_t,a_t) \mathbbm{1}(s_t,a_t = s,a).
\end{align*}
The exploration bonus is defined by aggregating the uncertainty on the reward and transition functions:
\begin{align}\label{eq:exploration.bonus}
\begin{split}
        b_k(s,a) := (c+\rmaxbound) \; &
                \underbrace{
\sqrt{\frac{7 \ln \left( 2SAt_k/ \delta\right)
                }{\max\lbrace 1, N_k(s,a)\rbrace}}
        }_{:= \beta_k^{sa}}
        + \frac{c}{N_k(s,a)+1},
\end{split}
\end{align}
where $\beta_k^{sa}$ is derived from Hoeffding-Azuma inequality.
The application of \scopt to the MDP defined by $(\calS,\A, \wb{p}_k, \wb{r}_k+b_k)$ will not lead to a solution of problem~\ref{eq:opt.superior.spanc} in general since none of the three assumptions of Prop.~\ref{prop:recap_scopt} is met.
To satisfy the first and second assumptions, we introduce MDP $\wh{M}_k := (\calS, \A, \wh{p}_k, \wh{r}_k)$ where $\wh{r}_k(s,a) = \wo{r}_k(s,a) + b_k(s,a)$ for all $(s,a) \in \calS \times \A$, $\wb{s} \in \calS$ is an arbitrary reference state and
\begin{equation}\label{E:hatp}
        \wh{p}_k(s'|s,a) = \frac{N_k(s,a) \wo{p}_k(s'|s,a)}{N_k(s,a) + 1} +   \frac{\mathbbm{1}(s'= \wb{s}) }{N_k(s,a) + 1}
\end{equation}
is a biased (but asymptotically consistent) estimator of the probability of transition $(s,a) \mapsto s'$.
To satisfy the third assumption, we define the \emph{augmented} MDP $\wh{M}_k^+ = (\calS, \A^+, \wh{p}_k^+, \wh{r}_k^+)$ obtained by duplicating every action in $\wh{M}_k$ with transition probability unchanged and reward set to $0$. Formally, $\A^+ = \A \times \{1, 2 \} $ and for the sake of clarity, any pair $(a,i) \in \A \times \{1,2\}$ is denoted by $a_i$. We then define $\wh{p}_k^+(s'|s,a_i) := \wh{p}_k(s'|s,a)$ and $\wh{r}_k^+(s,a_i) := \wh{r}_k(s,a) \cdot \mathbbm{1}(i=1)$. In the next section we will verify that $\wh{M}_k^+$ satisfies all the assumptions of Prop.~\ref{prop:recap_scopt}.
Note that the policy $\pi_k$ returned by \scopt takes action in the \emph{augmented} set $\A^+$. The projection on $\A$ is simply $\pi_k(s,a) \leftarrow \pi_k(s,a_1) + \pi_k(s,a_2)$, for all $s \in \calS, a \in \A$ (we use the same notation for the two policies). $\pi_k$ is executed until the episode ends \ie until the number of visits in at least one state-action pair has doubled (see Fig.~\ref{fig:ucrl.constrained}).\\
\textbf{Remark.}
\bonusscal only requires to plan on an empirical MDP with exploration bonus rather than an extended MDP (with continuous action space).
This removes the burden of computing the best probability in a confidence interval 
which has a worst-case computational complexity linear in the number of states $S$ \citep[Sec. 3.1.2]{Jaksch10}.
Therefore, \bonusscal is not only simpler to implement but also less computationally demanding. Furthermore, removing the optimistic step on the transition probabilities allows the exploration bonus scheme to be easily adapted to any MDP that can be efficiently solved (e.g., continuous smooth MDPs).

\input{algo_bonus}

\subsection{Requirements for \scopt}

We show that the three assumptions of Prop.~\ref{prop:recap_scopt} required from \scopt to solve~\eqref{eq:opt.superior.spanc} for $\wh{M}_k^+$ are satisfied. The arguments are similar to those used by \citet[Sec. 6]{fruit2018constrained} for \scal.
We denote by $\wh{L}^+$, $\wh{L}$ and $L$ the optimal Bellman operators of $\wh{M}_k^+$, $\wh{M}_k$ and $M^*$ respectively. Similarly, we denote by $\wh{T}_{c}^+$, $\wh{T}_{c}$ and $\opT{}$ the truncated Bellman operators (Eq.~\ref{eq:opT}) of $\wh{M}_k^+$, $\wh{M}_k$ and $M^*$ respectively.

\textbf{Contraction.}
The small bias in the definition of $\wh{p}_k$ ensures that the ``\emph{attractive}'' state $\wb{s}$ is reached with non-zero probability from any state-action pair $(s,a_i)$ implying that the \emph{ergodic coefficient} of $\wh{M}_k^+$ defined as 
$
\gamma_k = 1 - \min_{
        \substack{
s,u \in \mathcal{S},\\ a, b \in \A^+
}
} 
\left\{ \sum_{j \in \calS}\min \left\{ \wh{p}_k(j|s,a),\wh{p}_k(j|u,b) \right\} \right\}
$
is smaller than $1-\min_{s,a}\left\{ \frac{N_k(s,a,\wb{s}) + 1}{N_k(s,a) + 1}\right\} <1$ and thus $\wh{L}^+$ (the Bellman operator of $\wh{M}_k^+$) is $\gamma_k$-contractive~\citep[Thm. 6.6.6]{puterman1994markov}. 

\textbf{Unichain.} By construction, the attractive state $\overline{s}$ necessarily belongs to all \emph{recurrent classes} of all policies implying that $\wh{M}_k^+$ is unichain (\ie all policies are unichain).

\textbf{Global feasibility.} Let $v \in \mathbb{R}^S$ such that $\SP{v} \leq c$ and let $(s^*,a^*_i) \in \calS \times \A^+$ be such that $\wh{r}_k^+(s^*,a^*_i) + \wh{p}_k^+(\cdot|s^*,a^*_i)^\transp v = \min_{s \in \calS} \left\{\max_{a \in \A^+} \{\wh{r}_k^+(s,a) + \wh{p}_k^+(\cdot|s,a)^\transp v \}\right\}$. For all $(s,a_2) \in \calS \times \A^+$ we have: 
\[
        \wh{p}_k^+(\cdot|s,a_2)^\transp v - \wh{p}_k^+(\cdot|s^*,a^*_i)^\transp v \leq \max_{s \in \calS}\{v(s)\} - \min_{s \in \calS}\{v(s)\} = \SP{v} \leq c
\]
and $\wh{r}_k^+(s,a_2) = 0 \leq \wh{r}_k^+(s^*,a^*_i)$. Therefore, for all $s \in \calS,~\min\limits_{a_j \in \A^+} \{\wh{r}_k^+(s,a_j) + \wh{p}_k^+(\cdot|s,a_j)^\transp v \} \leq \min_{s'} \big\{\wh{L}^+v(s')\big\} + c$ implying that $\wh{T}_{c}^+$ is globally feasible at $v$.


\subsection{Optimistic Exploration Bonus}\label{sec:exploration_bonus}

All algorithms relying on the OFU principle (\eg \ucrl, \optpsrl, \scal, etc.) have the property that the optimal gain of the MDP used for planning is an upper bound on the optimal gain of the true MDP $g^*$. This is a key step in deriving regret guarantees.
If we want to use the same proof technique for \bonusscal, we also have to ensure that the policy $\pi_k$ is \emph{gain-optimistic} (up to an $\varepsilon_k$- accuracy), \ie $\wh{g}^+_k := g^*_c\left(\wh{M}_k^+ \right) \geq g^* $.
The exploration bonus was tailored to enforce this property.
To prove gain-optimism we rely on the following proposition which is a direct consequence of  \citet[Lem. 8]{fruit2018constrained}:
\begin{proposition}[Dominance]\label{prop:bound.gain}
If there exists $(g,h)$ satisfying $\wh{T}_{c}^+{h} \geq h + ge$ then $\wh{g}^+_k \geq g$.
\end{proposition}
\begin{proof}
By induction, using the monotonicity and linearity of $ \wh{T}_{c}^+$ \citep[Lemma 16 (a) \& (c)]{fruit2018constrained}, we have that
$
    \forall n \in \mathbb{N},~~\big(\wh{T}_{c}^+\big)^{n+1} h \geq \big(\wh{T}_{c}^+\big)^{n} h + ge
$.
By Prop.~\ref{prop:recap_scopt}, $\lim_{n \to + \infty} \big(\wh{T}_{c}^+\big)^{n+1} h - \big(\wh{T}_{c}^+\big)^{n} h = \wh{g}^+_k$. Taking the limit when $n$ tends to infinity in the previous inequality yields: $\wh{g}^+_k \geq g$.
\end{proof}
Recall that the optimal gain and bias of the true MDP $(g^*,h^*)$ satisfy the optimality equation $Lh^* = h^* + g^* e$ (Sec.~\ref{sec:mdp}). Since in addition $\SP{h^*} \leq c$ (by assumption), we also have $\SP{L h^*} = \SP{h^* + g^* e} = \SP{h^*} \leq c$ and so $\opT{h^*} = L h^*$.
According to Prop.~\ref{prop:bound.gain}, it is sufficient to show that $\wh{T}_{c}^+ h^* \geq h^* + g^* e = \opT{h^*} $ to prove optimism.
\citet[Lemma 15]{fruit2018constrained} also showed that the span projection $\proj{}$ (see Eq.~\ref{eq:opT}) is monotone implying that a sufficient condition for $\proj{\wh{L}^+h^*} = \wh{T}_{c}^+ h^* \geq \opT{h^*} = \proj{\wh{L}h^*}$ to hold is to have $\wh{L}^+ h^* \geq L{h^*} $. 
With our choice of bonus, this inequality holds with high probability (w.h.p) as a consequence of the following lemma:
\begin{lemma}\label{lem:explbonus.inequality}
        For all $T \geq 1$ and $k \geq 1$, with probability at least $1-\frac{\delta}{15 t_k^6}$, for any $(s,a) \in \calS \times \A $ we have:
$
                \wb{r}_k(s,a) + b_{k}(s,a) + \wh{p}_k(\cdot|s,a)^\intercal h^* \geq  {r}(s,a) + {p}(\cdot|s,a)^\intercal h^* 
$.
\end{lemma}
\begin{proof}
 Hoeffding-Azuma inequality implies that with probability at least $1-\frac{\delta}{15 t_k^6}$, for all $k\geq 1$ and for all pairs $(s,a) \in \mathcal{S}\times \mathcal{A}$, $ \left|\wo{r}_k(s,a) - r(s,a)\right| \leq \rmaxbound \beta_{k}^{sa}$ and $\left|\left(\wo{p}_k(\cdot|s,a) - p(\cdot|s,a)\right)^\intercal h^*\right| \leq c \;\beta_{k}^{sa}$.
Finally, we also need to take into account the small bias introduced by $\wh{p}_k(\cdot|s,a)$ compared to $\wo{p}_k(\cdot|s,a)$ which is not bigger than ${c}/{(N_k(s,a)+1)}$ by definition.
\end{proof}
Denote by $\wh{L}$ the optimal Bellman operator of $\wh{M}_k$. A direct implication of Lem.~\ref{lem:explbonus.inequality} is that $\wh{L} h^* \geq L{h^*}$ w.h.p.
 Since by definition $\wh{p}_k^+(s'|s,a_1) = \wh{p}_k^+(s'|s,a_2) = \wh{p}_k(s'|s,a) $ and $\wh{r}_k^+(s,a_2) \leq  \wh{r}_k^+(s,a_1) $ it is immediate to see that $\wh{L}^+ h^* = \wh{L} h^*$ implying that $\wh{L}^+ h^* \geq L{h^*}$ w.h.p. As a result, we have the following desired property:
\begin{lemma}\label{L:optimism}
        For all $T \geq 1$ and $k \geq 1$, with probability at least $1-\frac{\delta}{15 t_k^6}$, $\wh{g}^+_k = g^*_c(\wh{M}_k) \geq g^*$.
\end{lemma}
\paragraph{Remark.} 
Note that the argument used in this section to prove optimism (Lem.~\ref{L:optimism}) significantly differs from the one used by \citet[\ucrl]{Jaksch10} and \citet[\scal]{fruit2018constrained}. \ucrl and \scal compute a (nearly) optimal policy of an \emph{extended} MDP that \emph{``contains''} the true MDP $M^*$ (w.h.p.). This immediately implies that the gain of the extended MDP is bigger than $g^*$ (analogue property of Lem.~\ref{L:optimism}). The main advantage of our argument compared to theirs is that it allows for a \emph{tighter} optimism. 
To see why, note that the exploration bonus quantifies by how much $\wh{L}^+ h^*$ is bigger than $L{h^*} $ and approximately scales as $b_k(s,a) =\wt{\Theta} \left( \max\{\rmaxbound,c\}/\sqrt{N_k(s,a)} \right)$. 
In contrast, \ucrl and \scal use an optimistic Bellman operator $\wt{L}$ such that $\wt{L}h^*$ is bigger than $Lh^*$ by respectively $\wt{\Theta}\left(\rmaxbound D\sqrt{\Gamma/N_k(s,a)}\right)$ (\ucrl) and $\wt{\Theta}\left(\max\{\rmaxbound,c\}\sqrt{\Gamma/N_k(s,a)}\right)$ (\scal). 
In other words, the optimism in \bonusscal is tighter by a multiplicative factor $\sqrt{\Gamma}$. 
A natural next step would be to investigate whether our argument could be extended to \ucrl and \scal in order to save $\sqrt{\Gamma}$ for the optimism. We keep this open question for future work.

\subsection{Regret Analysis of \bonusscal}
\label{app:regret.bonusscal}
We now give the main result of this section:
\begin{theorem}\label{thm:regret.bonusscal}
 For any \emph{weakly communicating} MDP $M$ such that $\SP{h^{*}_M}\leq c$, with probability at least $1-\delta$ it holds that for any $T\geq 1$, the regret of \bonusscal is bounded as
 \begin{align*}
         \Delta(\bonusscal,T) = O \left( \max \lbrace \rmaxbound, c \rbrace \left( \sqrt{\nextstates S A T \ln \left( \frac{T}{\delta} \right)} + S^2A\ln^2\left(\frac{T}{\delta} \right) \right) \right).
 \end{align*}
\end{theorem}
%
Unlike \scal, \bonusscal does not have a regret scaling with $\min\{\rmaxbound D, c\}$ implying that whenever $c> D$, \bonusscal performs worse than \ucrl.
\scal builds an extended MDP that contains the true MDP and therefore the shortest path in the extended MDP is shorter than the shortest path in the true MDP implying that $\SP{\wt{v}_k} \leq \rmaxbound D$ with $\wt{v}_k$ being the solution returned by extended value iteration (Thm. 4 of \citet{Bartlett2009regal}).
Let $v_k$ be the solution returned by \scopt on $\wh{M}_k^+$, it is not clear how to bound $ \SP{v_k}$ other than using the prior knowledge $c$. This open question seems a lot related to the one of Sec.~\ref{sec:exploration_bonus} (\ie how to have a tighter optimism in \ucrl) and we also keep it for future work.

\textbf{Proof sketch.}
We now provide a sketch of the main steps of the proof of Thm.~\ref{thm:regret.bonusscal} (the full proof is reported in App.~\ref{app:regret.bonusscal}). In order to preserve readability, in the following, all inequalities should be interpreted up to minor approximations and in high probability.
Let $\nu_k(s,a)$ be the number of visits in $(s,a)$ during episode $k$ and $m$ be the total number of episodes. 
By using the optimism proved in Sec.~\ref{sec:exploration_bonus}, we can decompose the regret as:
\begin{align}\label{eq:regret.step1}
\Delta(\bonusscal, T) 
\lesssim \sum_{k=1}^m \sum_{s,a} \nu_k(s,a) \left( g_k - \sum_{a} r(s,a) \pi_k(s,a) \right)
\end{align}
where $g_k = 1/2 (\max\{\wh{T}^+_kv_k - v_k\} + \min\{\wh{T}^+_k v_k - v_k\})$ and $(v_k, \pi_k)$ is the solution of \scopt.
The stopping condition of \scopt applied to $\wh{M}_k^+$ is such that (after manipulation):
$g_k \leq \sum_a \pi_k(s,a) \left( \wh{r}_k(s,a) + \wh{p}_k(\cdot|s,a)^\transp v_k \right) - v_k(s) + \varepsilon_k.$
By plugging this inequality into~\eqref{eq:regret.step1} we obtain two terms: $\wb{r}_k(s,a) - r(s,a) + b_k(s,a)$ and $(\wh{p}_k(\cdot|s,a) - e_s)^\transp v_k$.
We can further decompose the scalar product as $(\wh{p}_k(\cdot|s,a) - p(\cdot|s,a))^\transp v_k + (p(\cdot|s,a) - e_s)^\transp v_k$.
The second terms is negligible in the final regret since it is of order $\wt{O}(c\sqrt{T} + cSA)$ when summed over $\calS$, $\A$ and episodes~\citep[][Eq. 56]{fruit2018constrained}.
On the other hand, the term $(\wh{p}_k(\cdot|s,a) - p(\cdot|s,a))^\transp v_k$ is the dominant term of the regret and represents the error of using the estimated $\wh{p}_k$ in place of $p$ in a step of value iteration.
As shown in Sec.~\ref{sec:exploration_bonus}, we can start bounding the error of using $\wh{p}_k$ in place of $\wb{p}_k$ by $c/(N_k(s,a)+1)$.
The remaining term is thus $(\wb{p}_k -p)^\transp v_k$.
Since $v_k$ depends on $\wb{p}_k$, we cannot apply Hoeffding-Azuma inequality as done in Sec.~\ref{sec:exploration_bonus} for the design of $b_k$.
Instead we use a worst-case approach and bound separately $\|\wb{p}_k(\cdot|s,a) - p(\cdot|s,a))\|_1 \lesssim \sqrt{\Gamma} \beta_{k}^{sa}$ and $\SP{v_k}\leq c$ which will introduce a $\sqrt{\Gamma}$ factor (by using Bernstein-Freedman inequality instead of Hoeffding-Azuma inequality). 
It is worth pointing out that $\Gamma$ only appears due to statistical fluctuations that we cannot control, and not from the optimism (\ie exploration bonus) that is explicitly encoded in the algorithm.
Concerning the reward, as shown in Sec.~\ref{sec:exploration_bonus}, we have that $|\wb{r}(s,a) - r(s,a)| \leq \rmaxbound \beta_{k}^{sa}$.
As a consequence, we can approximately write that:
\begin{align*}
        \Delta(\bonusscal,T) \lesssim \sum_{k=1}^m \sum_{s,a} \nu_k(s,a) \pi_k(s,a) \Big( b_k(s,a) + \underbrace{(c \sqrt{\Gamma} + \rmaxbound) \beta_{k}^{sa} + \sfrac{c}{(N_k(s,a)+1)}}_{:= d_k(s,a)} \Big)
\end{align*}
The proof follows by noticing that $d_k(s,a) + b_k(s,a) \leq 2 d_k(s,a)$, thus all the remaining terms can be bounded as in~\citep{fruit2018constrained}.

\textbf{Remarks.} Given the fact that the optimism in \bonusscal is tighter than in \scal by a factor $\sqrt{\Gamma}$, one might have expected to get a regret bound scaling as $c \sqrt{SAT}$ instead of $c \sqrt{S\Gamma AT}$, thus matching the lower bound of \citet{Jaksch10} as for the dependency in $S$.\footnote{From an algorithmic perspective we achieve the optimal dependence on $S$, although this is not reflecting in the regret bound.}
Unfortunately, such a bound seems difficult to achieve with \bonusscal (and even \scal) for the reason explained above.

On the other side,~\citep{pmlr-v70-azar17a,kakade2018variance} achieved such an optimal dependence in finite-horizon problems.
The main issue in extending such results is the different definition of the regret: their regret is defined as the difference between the value function at episode $k$ and the optimal one.
It is not clear how to map their definition to ours without introducing a linear term in $T$.
Concerning infinite-horizon undiscounted problems, \citet{Agrawal2017posterior} claimed to have obtained the optimal dependence in their optimistic posterior sampling approach.
To achieve such goal, they exploited the fact that  $|(\wb{p}_k(\cdot|s,a) - p(\cdot|s,a))^\transp \wt{v}_k| \lesssim \rmaxbound D\beta_{k}^{sa}$.
Unfortunately, as explained above, it is not possible to achieve such tight concentration by using a worst-case argument, as they do. As a result, optimistic PSRL would have a regret scaling as $D\sqrt{S\Gamma AT}$, while the improved bound in~\citep{Agrawal2017posterior} should be rather considered as a conjecture.\footnote{The problem has been acknowledged by the authors via personal communication.}

%% file: algo_bonus.tex
\begin{figure}[t]
\renewcommand\figurename{\small Figure}
\begin{minipage}{\columnwidth}
\bookboxx{
        \textbf{Input:} Confidence $\delta \in ]0,1[$, $r_{\max}$, $\calS$ ($\mathcal{I}$ for \scalcont), $\A$, $c \geq 0$ (and $L$ and $\alpha$ for \scalcont)


\noindent \textbf{For} episodes $k=1, 2, ...$ \textbf{do}

\begin{enumerate}[leftmargin=4mm,itemsep=0mm]
\item Set $t_k = t$ and episode counters $\nu_k (s,a) = 0$.

\item Compute estimates $\wh{p}_k^+(I(s') | I(s),a)$, $\wh{r}_k^+(I(s),a)$, $b_k(I(s),a)$ (Eq.~\ref{eq:exploration.bonus} or~\ref{eq:exploration.bonus.cont}) and build the MDP $\wh{M}_k^+$ (\bonusscal) or $\wh{M}^{ag+}_k$ (\scalcont).

\item Compute an $\rmaxbound/\sqrt{t_k}$-approximate solution of Eq.~\ref{eq:opt.superior.spanc} on $\wh{M}_k^+$ (\bonusscal) or $\wh{M}^{ag+}_k$ (\scalcont)

%
\item Sample action $a_t \sim \pi_k(\cdot|I(s_t))$.

\item \textbf{While} $\nu_k(I(s_t),a_t) < \max\{1, N_k(I(s_t),a_t)\}$ \textbf{do}
\begin{enumerate}[leftmargin=4mm,itemsep=-1mm]
        \item Execute $a_t$, obtain reward $r_{t}$, and observe next state $s_{t+1}$.
        \item Set $\nu_k (s_t,a_t) \pluseq 1$.
        \item Sample action $a_{t+1} \sim \pi_k(\cdot|I(s_{t+1}))$ and set $t \pluseq 1$.
\end{enumerate}

\item Set $N_{k+1}(s,a) = N_{k}(s,a)+ \nu_k(s,a)$.
\end{enumerate}
}
 \vspace{-0.1in}
 \caption{\small Structure of \bonusscal and \scalcont. For \bonusscal by definition we have $I(s) = s$.}
\label{fig:ucrl.constrained}
\end{minipage}
\vspace{-0.2in}
\end{figure}

%% file: sec_continuous.tex
\section{\scalcont: \bonusscal for continuous state space}
\label{S:exploration.bonus.continuous}
We now consider an MDP with continuous state space $\calS = [0,1]$ and discrete action space $\A$.
In general, it is impossible to approximate an arbitrary function with only a finite number of samples. As a result, we introduce the same smoothness assumption as \citet{DBLP:journals/corr/abs-1302-2550} (H\"older continuity):
\begin{assumption}\label{asm:continuous.mdp}
There exist $L, \alpha > 0$ s.t. for any two states $s,s' \in \calS$ and any action $a\in \A$:
\begin{equation*}
 |r(s,a) - r(s',a)| \leq \rmaxbound L |s -s'|^{\alpha} ~~~ \text{ and } ~~~ \|p(\cdot|s,a) - p(\cdot|s',a)\|_1 \leq L |s-s'|^{\alpha}
\end{equation*}
\end{assumption}
Similarly to Sec.~\ref{S:exploration.bonus} we start presenting \scalcont, the variant of \bonusscal for continuous state space, and 
then we provide its theoretical guarantees (see Sec.~\ref{sec:cont.regret}).

\subsection{The algorithm}
In order to apply \bonusscal to a continuous problem, a natural idea is to discretize the state space as is done by \citet{DBLP:journals/corr/abs-1302-2550}.
We therefore partition $\calS$ into $S$ intervals defined as $I_1 := [0, \frac{1}{S}]$ and $I_k = (\frac{k-1}{S}, \frac{k}{S}]$ for $k=2,\ldots, S$.
The set of \emph{``aggregated''} states is then $\mathcal{I} := \{I_1, \ldots, I_S\}$ ($|\mathcal{I}| = S$). As can be expected, we will see that the number of intervals $S$ will play a central role in the regret.
Note that the terms $N_k(s,a,s')$ and $N_k(s,a)$ defined in Sec.~\ref{S:exploration.bonus} are still well-defined for $s$ and $s'$ lying in $[0,1]$ but are $0$ except for a finite number of $s$ and $s'$ (see Def.~\ref{def:estimated.cont.mdp}).
For any subset $I\subseteq \mathcal{S}$, the sum $\sum_{s \in I} u_{s} $ is also well-defined as long as the collection $\left(u_{s} \right)_{s \in I}$ contains only a finite number of non-zero elements. We can therefore define the \emph{aggregated} counts, rewards and transition probabilities for all $I,J \in \mathcal{I}$ as: $N_k(I,a) := \sum_{s \in I} N_k(s,a)$, 
\begin{align*}
%
\wb{r}^{ag}_k(I,a) := \frac{1}{N_k(s,a)} \sum_{t=1}^{t_k -1} r_t(s_t,a_t) \mathbbm{1}(s_t \in I,a_t = a), ~~
\wb{p}^{ag}_k(J|I, a) := \frac{\sum_{s' \in J} \sum_{s \in I} N_k(s,a,s')}{\sum_{s \in I} N_k(s,a)}.
\end{align*}
Similarly to the discrete case, we define the exploration bonus of an aggregated state as
\begin{align}\label{eq:exploration.bonus.cont}
\begin{split}
        b_k(I,a) := &(c +\rmaxbound) \left( \beta_k^{Ia} + L S^{-\alpha} \right) + \frac{c}{N_k(I,a)+1}
\end{split}
\end{align}
%
While $\beta_k^{Ia}$ is defined as in~\eqref{eq:exploration.bonus} on the discrete ``aggregated'' MDP, the main difference with the discrete bonus~\eqref{eq:exploration.bonus} is an additional $O(cLS^{-\alpha})$ term that accounts for the fact that the states that we aggregate are not completely identical but have parameters that differ by at most $LS^{-\alpha}$.
We pick an arbitrary reference aggregated state $\wb{I}$ and define $\wh{M}^{ag}_k = (\mathcal{I}, \A, \wh{p}_k^{ag}, \wh{r}_k^{ag})$ the ``aggregated'' (discrete) analogue of $\wh{M}_k$ defined in Sec.~\ref{S:exploration.bonus}, where $\wh{r}_k^{ag} = \wb{r}^{ag}_k + b_k$ and
\[
        \wh{p}^{ag}_k(J|I,a) := \frac{N_k(I,a) \wb{p}^{ag}_k(J|I,a)}{N_k(I,a) + 1} + \frac{\mathbbm{1}(J=\wb{I})}{N_k(I,a) + 1},
\]
Similarly we ``augment'' $\wh{M}^{ag}_k$ into $\wh{M}^{ag+}_k = (\mathcal{I}, \A^+, \wh{p}^{ag+}_k, \wh{r}^{ag+}_k)$ (analogue of $\wh{M}_k^+$ in Sec.~\ref{S:exploration.bonus}) by duplicating each transition in $\wh{M}^{ag}_k$ with the transition probability unchanged and the reward set to $0$.

At each episode $k$, \scalcont uses \scopt (with the same parameters as in Sec.~\ref{S:exploration.bonus}) to solve optimization problem \eqref{eq:opt.superior.spanc} on $\wh{M}^{ag+}_k$.
This is possible because although the state space of $M^*$ is infinite, $\wh{M}^{ag+}_k$ has only $S < +\infty$ states.
\scopt returns an \emph{optimistic} (nearly) optimal policy $\pi_k$ 
satisfying the span constraint.
This policy is defined in the aggregated discrete state space but can easily be extended to the continuous case as ${\pi}_k(s,a) = {\pi}_k(I(s),a)$ for any $(s,a)$.
Policy ${\pi}_k$ is executed until the end of the episode (see Alg.~\ref{fig:ucrl.constrained}).

\subsection{Regret Analysis of \scalcont}
\label{sec:cont.regret}
This section is devoted to the analysis of \scalcont. We start providing the regret bound:
\begin{theorem}\label{thm:regret.scalcont}
        For any continuous MDP $M$ with state space $\calS \in [0,1]$ and $A$ actions such that $\SP{h^*_M} \leq c$, with probability at least $1 - \delta$ it holds that for any $T \geq 1$, the regret of \scalcont is bounded as
        \[
                \Delta(\scalcont,T) = O
                \left( \max{\left\lbrace \rmaxbound,c \right\rbrace} \left( S\sqrt{A T \ln\left(\frac{T}{\delta}\right)}
                        + S^2A\ln^2\left(\frac{T}{\delta} \right) 
        + L S^{-\alpha}T \right) \right)
        \]
        For $T \geq L^{\sfrac{2}{\alpha}} A$ and by setting $S = \left(\alpha L \sqrt{\frac{T}{A}}\right)^{\sfrac{1}{(\alpha + 1)}}$, the regret is bounded w.h.p. as
        \[
                \Delta(\scalcont,T) = \wt{O} \left(\max\{\rmaxbound,c\} L^{\sfrac{1}{(\alpha+1)}} A^{\sfrac{\alpha}{(2\alpha+2)}} T^{\sfrac{(\alpha+2)}{(2\alpha+2)}} 
\right) 
        \]
\end{theorem}
Thm.~\ref{thm:regret.scalcont} shows that \scalcont achieves the same regret of \uccrl~\citep{DBLP:journals/corr/abs-1302-2550} while being the only implementable algorithm for regret minimization in this setting. 
It is worth to mention that it is possible to exploit the recent advances in the literature~\citep[see][]{fruit2018constrained} in order to derive an implementable variant of \uccrl (based on \scal).
However, this new algorithm will still require to plan on an extended MDP making it less computation efficient than \scalcont, while having the same regret bound.
Moreover, note that the analysis can be extended to the general $d$-dimensional case. As pointed out by~\citep{DBLP:journals/corr/abs-1302-2550}, $S^d$ intervals are used for the discretization leading to a regret bound of $\wt{O}(T^{\sfrac{(2d+\alpha)}{(2d+2\alpha)}})$ when $S = T^{\sfrac{1}{(2d+2\alpha)}}$.
Finally, we believe that \scalcont can be extended to the setting considered by~\citep{pmlr-v37-lakshmanan15} where, in addition to H\"{o}lder conditions, the transition function is assumed to be $\kappa$-times smoothly differentiable.
In the case of Lipschitz model, \ie $\alpha=1$, this means that it is possible obtain an asymptotic regret (as $\kappa \to \infty$) of $\wt{O}(T^{2/3})$ while \scalcont is achieving $\wt{O}(T^{3/4})$.
We leave the derivation of this variant for future work.

\textbf{Proof sketch.}
The continuous case considered in this section can be interpreted as a generalization of the discrete case, thus presenting more challenges.
The \emph{main technical challenge} is to be able to compare the solution of problem~\eqref{eq:opt.superior.spanc} on $\wh{M}^{ag+}_k$ (discrete state space MDP) with the solution of~\eqref{eq:opt.superior.spanc} on $M^*$ (continuous state space MDP) and thus prove the optimism.
We start introducing an intermediate empirical continuous MDP $\wh{M}_k$ that will be used in the rest of the proof.
\begin{definition}[Estimated continuous MDP]\label{def:estimated.cont.mdp}
        Let $\wh{M}_k = (\calS,\A, \wh{p}_k, \wh{r}_k)$ be the continuous state space MDP \st for all $(s,a) \in \calS \times \A$
        \begin{align*}
                \wh{r}_k(s,a) &:= \wb{r}_k(s,a) + b_k(s,a) = \wb{r}_k^{ag}(I(s), a) + b_k(I(s),a)\\
            \wh{p}_k(s'|s,a) &:= \frac{N_k(I(s),a) \wb{p}_k(s'|s,a)}{N_k(I(s),a)+1} + \frac{S \cdot \mathbbm{1}(s' \in I(\wb{s}))}{N_k(I(s),a) + 1}.
        \end{align*}
                where $I : \calS \to \mathcal{I}$ is the function mapping  a state $s$ to the interval containing $s$ and the term $\wb{p}_k(s'|s,a)$ is the Radon-Nikodym derivative of the cumulative density function $F(s) = \sum_{s' \leq s} \frac{\sum_{x \in I(s)}N_k(x,a,s')}{N_k(I(s),a)}$, meaning that for any measurable function $f$ and any measurable set $Z \subseteq [0,1]$, $\int_Z \wb{p}_k(s'|s,a) f(s') \mathrm{d}s' = \sum_{s' \in Z} \frac{\sum_{x \in I(s)}N_k(x,a,s')}{\sum_{x \in I(s)}N_k(x,a)} f(s')$.
\end{definition}
This MDP is one of the possible instances of continuous MDP that, when aggregated over the interval set $\mathcal{I}$, matches the discrete MDP $\wh{M}_k^{ag}$.\footnote{It seems that an intermediate (\emph{extended}) continuous MDP is also used in the proof of \uccrl but never formally defined, leaving to the reader the need of interpreting the properties of this MDP. Due to the lack of rigorous definition, few steps in the regret proof of \uccrl are not completely clear.}
In particular, by definition, $\forall J \in \mathcal{I}$, $\int_{J} \wb{p}_k(s'|s,a) \mathrm{d}s' = \wb{p}_k^{ag}(J|s,a) := \wb{p}_k^{ag}(J|I(s),a)$ and $\forall (s,J) \in \calS \times \mathcal{I}$:
\begin{equation}\label{eq:integral.p.perturbed}
        \begin{aligned}
                \int_{J} \wh{p}_k(s'|s,a) \mathrm{d}s'
                &= \int_{J} \frac{N_k(I(s),a) \wb{p}_k(s'|s,a)}{N_k(I(s),a)+1} \mathrm{d}s' + 
                    \frac{
                    S \int_J \mathbbm{1}(s' \in I(\wb{s}))\mathrm{d}s'
                    }{N_k(I(s),a) + 1}
                = \wh{p}_k^{ag}(J|I(s),a)
        \end{aligned}
\end{equation}
We leverage this definition to prove an analogous of Lem.~\ref{lem:explbonus.inequality} for the continuous case.
\begin{lemma}\label{lem:explbonus.inequality.cont}
        For all $T \geq 1$ and $k \geq 1$, with probability at least $1-\frac{\delta}{15 t_k^6}$, for any $(s,a) \in \calS \times \A $ we have:
$
\underbrace{\wb{r}_k(s,a) + b_{k}(I(s),a)}_{:=\wh{r}_k(s,a)} + \int_{\calS} \wh{p}_k(s'|s,a) h^*(s') \mathrm{d}s' \geq  {r}(s,a) + \int_{\calS} p(s'|s,a) h^*(s') \mathrm{d}s'
$
\end{lemma}
\begin{proof}
The main and crucial difference in the proof is that due to the aggregation of states, $\wb{p}_k$ and $\wb{r}_k$ do not statistically concentrate around the true values $p$ and $r$. To overcome this problem we decompose $\wh{p}_k -p$ into the sum of three terms $(\wh{p}_k -\wo{p}_k) + (\wo{p}_k - \wt{p}_k) +  (\wt{p}_k -p)$ with $\wt{p}_k(s'|s,a) :=\frac{1}{N_k(I(s),a)}\sum_{x \in I(s)} N_k(x,a) p(s'|x,a)$. We show that $\int_{\calS}(\wt{p}_k(s') -p(s'))h^*(s')\mathrm{d}s' = O(LS^{-\alpha})$ (Asm.~\ref{asm:continuous.mdp}: smoothness assumption) while $\int_{\calS}(\wh{p}_k(s') -\wo{p}_k(s')) h^*(s')\mathrm{d}s' = O(1/N_k)$ (biased estimator). 
The term $\int_{\calS}(\wo{p}_k(s') - \wt{p}_k(s')) h^*(s')\mathrm{d}s'$ can be bounded using concentration inequalities but requires more work than in the discrete case. In the discrete case, for a given state-action pair $(s,a) \in \calS \times \A$, the difference $(\wo{p}_k(\cdot|s,a) -\wt{p}_k(\cdot|s,a))^\intercal h^*$ is usually interpreted as the deviation of a sum of independent random variables from its expectation \citep[Section 4.4]{bandittorcsaba} and can be bounded using Hoeffding inequality. 
Since there is only a finite number of possible $(s,a)$, it is possible to take a union bound over state-action pairs. In the continuous case, the difference $\int_{\calS}(\wo{p}_k(s') - \wt{p}_k(s')) h^*(s')\mathrm{d}s'$ does not just depend on a single $s$ but on the (random) set of states belonging to a given interval $I \in \mathcal{I}$ that have been visited. 
There is an uncountable number of possible such sequences of states and so we can not use a union bound argument. 
Instead, we rely on a martingale argument and Azuma inequality for the proof. We decompose $\wh{r}_k -r$ as $(\wh{r}_k -\wt{r}_k) + (\wt{r}_k -r)$ with $\wt{r}_k(s,a) := \frac{1}{N_k(I(s),a)} \sum_{x \in I(s)} N_k(x,a) r(x,a)$ and proceed similarly for the reward.
The detailed proof can be found in App.~\ref{app:continuous} (Lem.~\ref{L:bonus.inequality.continuous}).
\end{proof}
Note that, as a consequence of Lem.~\ref{lem:explbonus.inequality.cont}, we have that $\wh{L} h^* \geq L h^*$ w.h.p.  where $\wh{L}$ is the optimal Bellman operator of the continuous empirical MDP $\wh{M}_k$.
This, together with Prop.~\ref{prop:bound.gain}, is sufficient to prove that the exploration bonus in~\eqref{eq:exploration.bonus.cont} makes $\wh{M}_k$ optimistic \wrt $M^*$: $g^*_c(\wh{M}_k) \geq g^*$ w.h.p.
We cannot directly extend this argument to prove optimism for $\wh{M}^{ag+}_k$ since the aggregated MDPs lie in a different state space.
The key property used in this setting is that the $n$-times application of $\wh{L}^{ag}$ (optimal Bellman operator of $\wh{M}_k^{ag}$) and $\wh{L}$ to a constant vector are identical.\footnote{In the App.~\ref{app:continuous} we show that $(\wh{L}^{ag})^n v_0 = (\wh{L})^n v_0$ for any piecewise constant vector over intervals in $\mathcal{I}$.}
As a consequence, we can prove that $\wh{M}_k^{ag+}$ is optimistic:
\begin{lemma}\label{L:optimism.cont}
        For all $T \geq 1$ and $k \geq 1$, with probability at least $1-\frac{\delta}{15 t_k^6}$, $\wh{g}^+_k = g^*_c(\wh{M}^{ag+}_k) \geq g^*$.
\end{lemma}
\begin{proof}
        We start noticing that, starting from a value $v_0(s) = 0$ ($\forall s$), the application of the Bellman operator of $\wh{M}^{ag}_k$ and $\wh{M}_k$ is such that $(\wh{L}^{ag})^n v_0(s) = (\wh{L})^n v_0(s)$, $\forall s \in \calS, n >0$.
This is due to the fact that $v_n$ is constant over any interval $I$ (for any $n$) and~\eqref{eq:integral.p.perturbed} holds:
\begin{align*}
   \forall n\geq 0, \quad \int_{s \in \calS} \wh{p}_k(s'|s,a) v_n(s') \mathrm{d}s' 
       &= \sum_{J \in \mathcal{I}} v_n(J) \int_{J} \wh{p}_k(s'|s,a)\mathrm{d}s' = \sum_{J \in \mathcal{I}} v_{n}(J)\; \wh{p}^{ag}_k(J| I(s),a)
\end{align*}
Then, $\forall n,s,~\wh{T}_c^{ag} v_n(s) = \wh{T}_c v_n(s)$ implying $g^*_c(\wh{M}^{ag}_k) = g^*_c(\wh{M}_k)$.
The optimism of $\wh{M}^{ag+}_k$ follows by the fact that the ``augmentation'' does not impact the gain: $g^*_c(\wh{M}^{ag+}_k) = g^*_c(\wh{M}^{ag}_k)$~\citep[][Lem. 20]{fruit2018constrained}.
The detailed proof can be found in App.~\ref{app:continuous} (Lem.~\ref{L:optimism.cont}).
\end{proof}
%
\textbf{Remark.} The proof of optimism does not seem as straightforward as suggested by~\citet{DBLP:journals/corr/abs-1302-2550} (regret proof of \uccrl).
They use an informal ``inclusion'' argument (\ie $M^*$ is included in the discretized extended MDP used for planning) which seems not easy to formally prove since the true and ``optimistic'' MDPs are of different nature (the true MDP has a continusous state space unlike the optimistic one which is discretized).
Overall, we believe that an additional contribution of this paper is to provide a more rigorous analysis of the continuous case compared to the existing literature.

We have now all the key properties to apply the same regret analysis stated for the discrete case (with several technical arrangements to deal with the continuous case) in order to prove Thm.~\ref{thm:regret.scalcont}.
The complete proof can be found in App.~\ref{app:proof.regret.continuous}.

%% file: sec_conclusion.tex
\section{Conclusion}
In this paper we provide the first analysis of exploration bonus in infinite-horizon undiscounted problems, a more challenging setting than the finite-horizon~\citep[see \eg][]{pmlr-v70-azar17a,kakade2018variance,qlearning2018} and discounted \citep{strehl2008analysis}.
Compared to these other settings, we also extended the analysis to the continuous case and we provided the first implementable and efficient (no need to plan on an extend MDP) algorithm.
We finally showed, through a formal derivation of the exploration bonus, that the empirical count-based exploration bonuses are in general not sufficient to provide optimism and thus prone to under-exploration. In particular, the knowledge of the span of the optimal bias function is required in order to properly scale the bonus.
Moreover, even in the finite-horizon case, the mentioned approaches use the knowledge of the horizon to scale the bonus. The planning horizon is in turn an upper-bound to the span of the optimial value function, thus they exploit the same prior knowledge required by \bonusscal and \scalcont.

We also provide the tightest level of optimism for OFU algorithms by achieving the optimal dependence in the bonus \wrt the state dimensionality (it cannot further reduced while preserving theoretical guarantees). Unfortunately, this tighter optimism does not imply a tighter bound leaving open the quest for closing the gap between lower and upper bound in infinite-horizon undiscounted settings.
Moreover, it is unclear to us if the exploration bonus can be extended to settings where no-prior knowledge of the span of the optimal bias function is available, \eg in communicating (see \ucrl~\citep{Jaksch10}) or weakly-communicating MDPs (see \tucrl~\citep{fruit2018truncated}). We leave this question for future work.
Finally, \scalcont requires to known the smoothness parameters in order to define the discretization of the state space. We believe that some effort should be spent in the direction of removing such prior knowledge, making the algorithm more adaptive.

%% file: app_regret_continuous.tex
\section{Continuous state MDPs: the analysis of \scalcont} \label{app:continuous}

In all this section we say that a function $v: ~s \in \calS \mapsto \Re$ is \emph{piecewise constant on $\mathcal{I}$} when $\forall J \in \mathcal{I},~\forall s,s' \in J$ we have $v(s) = v(s') $ and we denote $v(J)$ the joint value.

\subsection{High probability bound using the exploration bonus (proof of Lem.~\ref{lem:explbonus.inequality})}\label{app:proba_bound_bonus}
To begin with, we introduce two variants of the exploration bonus that will be used for the regret proof:
\begin{align}
        \label{eq:tigher.expl.bonus}
        b_k(J,a) &:= c \cdot \min \left\{ \beta_{p,k}^{Ja} + \frac{1}{N_k(J,a) + 1}; 2\right\} + \min \left\{ \beta_{r,k}^{Ja}; \rmaxbound\right\} + (c + \rmaxbound)L S^{-\alpha}\\
        \label{eq:loser.expl.bonus}
        d_k(J,a) &:= c \cdot \min \left\{ \phi_{p,k}^{Ja} + \frac{1}{N_k(J,a) + 1}; 2\right\} + \min \left\{ \beta_{r,k}^{Ja}; \rmaxbound\right\} + (c + \rmaxbound)L S^{-\alpha}
\end{align}
where $\beta_{p,k}^{Ja} = \beta_k^{Ja}$ (see Eq.~\ref{eq:exploration.bonus} for the definition of $\beta_k^{sa}$), $\beta_{r,k}^{Ja} = \rmaxbound \beta_k^{Ja}$ and
\begin{align}\label{eq:bernstein.ci.cont}
        \phi_{p,k}^{Ja} :=  \sqrt{ \frac{7 S \lnn{\frac{3 S A t_k}{\delta}}}{N_k^+(J,a)}  } + \frac{14 S}{N_k^+(J,a)} \lnn{\frac{3 S At_k}{\delta}} \geq \beta_{p,k}^{Ja} 
\end{align}
with $N_k^+(J,a) := \max \{1,N_k(J,a)\}$.
Note that the Eq.~\ref{eq:tigher.expl.bonus} is a slightly tighter version of the exploration bonus considered in Eq.~\ref{eq:exploration.bonus}.
We define for all $(s,a)\in \calS \times \A$, $\wb{r}_k(s,a):= \wb{r}_k^{ag}(I(s), a)$.
We state a slightly more general result than Lem.~\ref{lem:explbonus.inequality}:
\begin{lemma}\label{L:bonus.inequality.continuous}
        Consider the estimated continuous MDP $\wh{M}_k$ defined in Def.~\ref{def:estimated.cont.mdp}. Let $(g^*,h^*)$ be a solution of the optimality equation $Lh^* = h^* + g^*$ such that $\SP{h^*}  \leq c$. For all $T\geq 1$ and $k\geq 1$, with probability at least $1-\frac{\delta}{15 t_k^6}$, for any $(s,a) \in \calS \times \A$ and for any function $v$ piecewise constant on $\mathcal{I}$ \st $\SP{v} \leq c$ we have:
        \begin{align*}
                &(a)~~ \quad b_{k}(s,a) \geq \left| \wb{r}_k(s,a) - r(s,a) + \int_{\calS} \left( \wh{p}_k(s'|s,a) - p(s'|s,a) \right) h^*(s') \mathrm{d}s' \right|\\
                &(b)~~ \quad d_k(s,a) \geq \left| \wb{r}_k(s,a) - r(s,a) + \int_{\calS} \left( \wh{p}_k(s'|s,a) - p(s'|s,a) \right) v(s') \mathrm{d}s' \right|
        \end{align*}
        where $b_k$ and $d_k$ are defined as in Eq.~\ref{eq:tigher.expl.bonus} and Eq.~\ref{eq:loser.expl.bonus}, respectively.
\end{lemma}

The rest of section \ref{app:proba_bound_bonus} is devoted to proving Lem.~\ref{L:bonus.inequality.continuous}. To do this, we introduce an intermediate continuous state-space MDP $\wt{M}_k := (\calS, \A, \wt{r}_k, \wt{p}_k)$ defined for all pairs $(s,a) \in \calS \times \A$ by:
\begin{align*}
 \wt{r}_k(s,a) := \frac{1}{N_k(I(s),a)} \sum_{x \in I(s)} N_k(x,a) r(x,a)\\
 \wt{p}_k(s'|s,a) := \frac{1}{N_k(I(s),a)}\sum_{x \in I(s)} N_k(x,a) p(s'|x,a)
\end{align*}
We decompose $\wh{p}_k -p$ and $\wb{r}_k -r$ as 
\begin{align}\label{eq:decomposition_p_r}
\wh{p}_k -p = (\wh{p}_k -\wo{p}_k) + (\wo{p}_k - \wt{p}_k) +  (\wt{p}_k -p) ~~ \text{and} ~~ \wb{r}_k -r= (\wb{r}_k -\wt{r}_k) + (\wt{r}_k -r)
\end{align}
and bound separetely all the terms. Similarly, we decompose $\wb{r}_k -r$ as $(\wb{r}_k -\wt{r}_k) + (\wt{r}_k -r)$.
We also define $w^*(s) := h^*(s) - \left( \inf\{ h^*(s)\} + \sup\{ h^*(s)\} \right)/2$ implying that for all $s \in \calS$, $w^*(s) \in [-c/2, c/2]$. 

\subsubsection{Bounding the difference between $\wt{r}_k/ \wt{p}_k$ and $r/p$}

To bound the differences~$\wt{r}_k(s'|s,a) - r(s,a)$ and $\int (\wt{p}_k(s'|s,a) - p(s'|s,a)) w^*(s') \mathrm{d}s'$ we simply use the smoothness assumption on the reward and transition model (see Asm.~\ref{asm:continuous.mdp}). For all $(s,a) \in \calS \times \A$ (using the triangle inequality):
\begin{align}\label{eq:pboundcontinuity}
\big|\wt{r}_k(s,a) - r(s,a)\big| \leq \frac{1}{N_k(I(s),a)} \sum_{x \in I(s)} N_k(x,a) \underbrace{\big| r(x,a) -r(s,a) \big|}_{\leq \rmaxbound L S^{-\alpha} \text{ since } x,s \in I(s)} \leq \rmaxbound L S^{-\alpha}
\end{align}
For the transition probability we have that for all $J\in \mathcal{I}$ (using the triangle inequality):
\begin{align}\label{eq:pboundcontinuity3}
\begin{split}
\sum_{J \in \mathcal{I}}\left| \int_{J} (\wt{p}_k(s'|s,a) - p(s'|s,a)) \mathrm{d}s' \right| &\leq  \sum_{J \in \mathcal{I}}\int_J |\wt{p}_k(s'|s,a) - p(s'|s,a) | \mathrm{d}s'\\
         &= \frac{1}{N_k(I(s),a)} \sum_{x \in I(s)} N_k(x,a)\int_\calS  \underbrace{|p(s'|x,a) - p(s'|s,a)|}_{\leq L S^{-\alpha} \text{ since } x,s \in I(s)} \mathrm{d}s'\\
  &\leq LS^{-\alpha}
\end{split}
\end{align}
and similarly:
\begin{align}\label{eq:pboundcontinuity2}
\begin{split}
\left|\int_\calS (\wt{p}_k(s'|s,a) - p(s'|s,a)) w^*(s') \mathrm{d}s' \right| &\leq c \int_\calS |\wt{p}_k(s'|s,a) - p(s'|s,a) | \mathrm{d}s'\\
         &= \frac{c}{N_k(I(s),a)} \sum_{x \in I(s)} N_k(x,a)\int_\calS  \underbrace{|p(s'|x,a) - p(s'|s,a)|}_{\leq L S^{-\alpha} \text{ since } x,s \in I(s)} \mathrm{d}s'\\
  &\leq cLS^{-\alpha}
\end{split}
\end{align}

\subsubsection{Bounding the difference between $\wh{p}_k$ and $\wo{p}_k$}\label{app:bound_diff_hat_bar}
Using the triangle inequality and the fact that $\int_{\calS} \mathbbm{1}(s' \in I(\wb{s})) \mathrm{d}s' = \int_{I(\wb{s})} 1 \mathrm{d}s' = \big| I(\wb{s}) \big| = 1/S $ we have that for any $(s,a) \in \calS \times \A$:
\begin{align}\label{eq:p.bound.perturbation.cont}
\begin{aligned}
\left|\int_{\calS} (\wh{p}_k(s'|s,a) - \wb{p}_k(s'|s,a)) w^*(s') \mathrm{d}s' \right| &\leq  \int_{\calS} \left|\wh{p}_k(s'|s,a) - \wb{p}_k(s'|s,a) \right| \cdot\left| w^*(s') \right|\mathrm{d}s'  \\
&= \left| \frac{N_k(I(s),a) }{N_k(I(s),a)+1} - 1 \right| \int_\calS \wb{p}_k(s'|s,a) \underbrace{|w^*(s')|}_{\leq c/2} \mathrm{d}s'\\ &~~~+ S\int_\calS \frac{|w^*(s')|\mathbbm{1}(s' \in I(\wb{s}))}{N_k(I(s),a) + 1} \mathrm{d}s'\\ 
&\leq \frac{c}{N_k(I(s),a) + 1}
\end{aligned}
\end{align}
and similarly:
\begin{align}\label{eq:p.bound.perturbation.cont2}
\begin{aligned}
\sum_{J\in \mathcal{I}}\left|\int_{J} (\wh{p}_k(s'|s,a) - \wb{p}_k(s'|s,a))  \mathrm{d}s' \right| \leq \int_{\calS} \left|\wh{p}_k(s'|s,a) - \wb{p}_k(s'|s,a)  \right|\mathrm{d}s' \leq \frac{1}{N_k(I(s),a) + 1}
\end{aligned}
\end{align}

\subsubsection{Bounding the difference between $\wt{r}_k/ \wt{p}_k$ and $\wo{r}_k/ \wo{p}_k$}\label{app:bound_diff_tilde_bar}

Let's consider a fixed pair $(s,a) \in \calS \times \A$ and a fixed aggregated state $J \in \mathcal{I}$. Our goal is to bound the differences $\int_{\calS}\big(\wt{p}_k(s'|s,a) - \wo{p}_k(s'|s,a)\big) w^*(s') \mathrm{d}s'$,~ $\int_{J}\wt{p}_k(s'|s,a) - \wo{p}_k(s'|s,a) \mathrm{d}s'$ and $\wt{r}_k(s,a) - \wo{r}_k(s,a)$. Since $\wt{p}_k$ and $\wt{r}_k$ are somehow the expected values of $\wo{p}_k$ and $\wo{r}_k$, we would like to use concentration inequalities. 
In the case of a finite state space $\calS$, \citet[\ucrl]{Jaksch10} and \citet[\scal]{fruit2018constrained} use concentration inequalities that apply to \emph{independent} random variables (r.v.). We argue that a more careful analysis is needed here since the states lie in an \emph{uncountable} set. Indeed, the implicit assumption made about the RL model for \ucrl and \scal is that for each state-action pair $(s,a)$, the rewards (respectively next states) are sampled from an infinite \emph{stack} of independent and identically distributed (i.i.d.) rewards (respectively next states). More precisely, each time the agent visits $(s,a)$, it receives a reward from the top of the stack of rewards associated to $(s,a)$ and moves to the state on the top of the stack of next states associated to $(s,a)$. The two samples are then withdrawn from their respective stacks (meaning that they cannot be popped again). For more details about why this is a valid model refer to \citep[Section 4.4]{bandittorcsaba}. In the case where $\calS$ and $\A$ are discrete sets (finite or countable), it is possible to use any concentration inequality for i.i.d. r.v. and then take a union bound over all ``stacks'' $(s,a)$ (and over rewards and next states). When $\calS$ is uncountable however, the same argument cannot be used (the probability of an uncountable union of events is not even always defined). Moreover, the terms $\wt{r}_k$ and $\wt{p}_k$ are obtained using sampled from different states $x \in I(s)$ instead of a single state $s$. To overcome these technical problems, we  use a variant of Doob's optional skipping \citep[\eg][Sec. 5.3, Lem. 4]{chow1988probability} and concentration inequalities for martingales (Azuma and Freedman inequalities). The theorem that we use (Thm.~\ref{T:inequality_stop_mds}) is formally proved in App.~\ref{app:prob.res}.

For any $t \geq 0$, the \sigalg induced by the past history of state-action pairs and rewards up to time $t$ is denoted $\mathcal{F}_t:= \sigma \left( s_1, a_1, r_1, \dots, s_t, a_t \right)$ where by convention $\mathcal{F}_0 = \sigma \left( \emptyset \right)$ and $\mathcal{F}_\infty := \cup_{t\geq 0}\mathcal{F}_{t}$. Let $\mathbb{F}$ denote the filtration $(\F_t)_{t\geq 0}$. We define the following adapted sequences and stopping times:

\textbf{1) Adapted sequences:}\\
We consider the following stochastic processes adapted to $\mathbb{F}$: $(w^*(s_t))_{t \geq 0}$, ~$(\one{s_t \in J})_{t \geq 0}$ and $(r_{t-1}(s_{t-1},a_{t-1}))_{t \geq 0}$ (with the conventions $ r_{-1}(s_{-1},a_{-1}) = r_0(s_0,a_0) = r_\infty(s_\infty,a_\infty) = 0$, ~$w^*(s_0) = w^*(s_\infty) = 0$ and $\one{s_0 \in J} = \one{s_\infty \in J} =0$). Theses processes are bounded as $|w^*(s_t)| \leq 2 \times \| w^* \|_{\infty} \leq c$, ~$|\one{s_t \in J}| \leq 1$ and $|r_{t-1}(s_{t-1},a_{t-1})| \leq \rmaxbound$ for all $t \geq 0$.

\textbf{2) Stopping times:}\\
 We define $\tau := \left(\tau_l\right)_{l \geq 0}$ \st $\tau_0 := 0$ and $\inf \{ t_k > t > \tau_l : s_t \in I(s), a_t = a \}$.
For all $l\geq 0$ and for all $t\geq 0$, $\tau_l :=\{ \tau_l = t \} \in \mathcal{F}_t$ and so $\tau_l$ is a stopping time \wrt $\mathbb{F}$ (see Def.~\ref{D:stop_time} in App.~\ref{app:prob.res}). By definition for any $l \geq 0$, $\tau_l < \tau_{l+1}$ \as (\ie $\tau$ is strictly increasing, see Lem.~\ref{L:stop_time_algebras2}). We denote $\G_l := \F_{\tau_{l\smallplus 1}}$ the \sigalg~at stopping time $\tau_{l+1}$ (see Def.~\ref{D:stop_time_algebras} in App.~\ref{app:prob.res}).
 
 All the assumptions of Thm.~\ref{T:inequality_stop_mds} are satisfied and so by taking a union bound and using the fact that $N_k(I(s),a) \leq t_k$ \as, we obtain that with probability at least $1-3\delta$ (simultaneously):
 \begin{align}
 \nonumber
  &\left| \sum_{l = 1}^{N_k(I(s),a)} \Big( r_{\tau_l}(s_{\tau_l},a_{\tau_l}) -\E\big[ r_{\tau_l}(s_{\tau_l},a_{\tau_l}) \big|\mathcal{G}_{l-1} \big]  \Big) \right| \leq \rmaxbound\sqrt{N_k(I(s),a) \lnn{\frac{2 t_k}{\delta}}}\\  
  \label{eq:bernstein_transition_proba}
  &\left| \sum_{l = 1}^{N_k(I(s),a)} \Big( \one{s_{\tau_l+1} \in J} - \E\big[ \one{s_{\tau_l+1} \in J} \big|\mathcal{G}_{l-1} \big] \Big) \right| \leq  2 \sqrt{V_{k}(J) \lnn{\frac{4 t_k}{\delta}}} + 4 \lnn{\frac{4t_k}{\delta}} \\ 
  \nonumber
  &\text{and}~~\left| \sum_{l = 1}^{N_k(I(s),a)} \Big( w^*(s_{\tau_l+1}) - \E\big[ w^*(s_{\tau_l+1}) \big|\mathcal{G}_{l-1} \big] \Big) \right| \leq c \sqrt{N_k(I(s),a) \lnn{\frac{2 t_k}{\delta}}} 
 \end{align}
 where $V_k(J) := \sum_{l=1}^{N_k(I(s),a)} \Varcond{\one{s_{\tau_l+1} \in J}}{\G_{l-1}}$.
 We now need to relate the above sums to $\int_{\calS}\big(\wt{p}_k(s'|s,a) - \wo{p}_k(s'|s,a) \big)w^*(s') \mathrm{d}s'$, $\int_{J}\wt{p}_k(s'|s,a) - \wo{p}_k(s'|s,a) \mathrm{d}s'$  and $\wt{r}_k(s,a) - \wo{r}_k(s,a)$. We also need to give an explicit formula for $V_k(J)$.
By defintion of $\tau$, we can rewrite $\wb{r}_k$ and $\wb{p}_k$ as follows:
\begin{align*}
   \wb{r}_k(s,a) &= \frac{1}{N_k(I(s),a)}\sum_{l=1}^{N_k(I(s),a)} r_{\tau_l}(s_{\tau_l},a_{\tau_l})\\
   \int_{J} \wb{p}_k(s'|s,a) \mathrm{d}s' &= \frac{1}{N_k(I(s),a)}\sum_{l=1}^{N_k(I(s),a)} \one{s_{\tau_l + 1} \in J}\\
   \int_{\calS} \wb{p}_k(s'|s,a) w^*(s') \mathrm{d}s' &= \frac{1}{N_k(I(s),a)}\sum_{l=1}^{N_k(I(s),a)} w^*(s_{\tau_l + 1})
\end{align*}
It is also easy to verify that the following holds: $\E\big[ w^*(s_{\tau_l+1}) \big|\mathcal{G}_{l-1} \big] = \int_{\calS} p(s'|s_{\tau_l},a_{\tau_l})w^*(s')  \mathrm{d}s' $, $\E\big[ \one{s_{\tau_l+1}\in J} \big|\mathcal{G}_{l-1} \big] = \int_{J} p(s'|s_{\tau_l},a_{\tau_l}) \mathrm{d}s' $ and $\E\big[ r_{\tau_l}(s_{\tau_l},a_{\tau_l}) \big|\mathcal{G}_{l-1} \big] = r(s_{\tau_l},a_{\tau_l})$ (see Lem.~\ref{L:conditioning_stop_time} in App.~\ref{app:prob.res} for a formal proof).
As a result, we can rewrite $\wt{r}_k$ and $\wt{p}_k$ as follows:
\begin{align*}
   \wt{r}_k(s,a) = \frac{1}{N_k(I(s),a)} \sum_{x \in I(s)} & N_k(x,a) r(x,a) = \frac{1}{N_k(I(s),a)}\sum_{l=1}^{N_k(I(s),a)} \E\big[ r_{\tau_l}(s_{\tau_l},a_{\tau_l}) \big|\mathcal{G}_{l-1} \big] \\ 
   \text{and}~~\int_{\calS} \wt{p}_k(s'|s,a) w^*(s') \mathrm{d}s' &= \frac{1}{N_k(I(s),a)}\sum_{x \in I(s)} N_k(x,a) \int_{\calS} p(s'|x,a)w^*(s')  \mathrm{d}s'\\
   &= \frac{1}{N_k(I(s),a)} \sum_{l=1}^{N_k(I(s),a)} \E\big[ w^*(s_{\tau_l+1}) \big|\mathcal{G}_{l-1} \big] \\
   \text{and similarly}~~\int_{J} \wt{p}_k(s'|s,a)  \mathrm{d}s' &= \frac{1}{N_k(I(s),a)} \sum_{l=1}^{N_k(I(s),a)} \E\big[ \one{s_{\tau_l+1} \in J} \big|\mathcal{G}_{l-1} \big]
\end{align*}
We can also give a more explicit expression for $V_k$:
\begin{align*}
 \Varcond{\one{s_{\tau_l +1} \in J}}{\G_{l-1}} &:= \E\big[ \underbrace{\one{s_{\tau_l+1}\in J}^2}_{=\one{s_{\tau_l+1}\in J}} \big|\mathcal{G}_{l-1} \big] - \E\big[ \one{s_{\tau_l+1}\in J} \big|\mathcal{G}_{l-1} \big]^2\\
 &= \int_{J} p(s'|s_{\tau_l},a_{\tau_l}) \mathrm{d}s' - \left(\int_{J} p(s'|s_{\tau_l},a_{\tau_l}) \mathrm{d}s'\right)^2
\end{align*}
implying:
\begin{align*}
 V_k(J) = \sum_{l=1}^{N_k(I(s),a)}  \underbrace{\left( 1 - \int_{J} p(s'|s_{\tau_l},a_{\tau_l}) \mathrm{d}s' \right)}_{\leq 1} \underbrace{\int_{J} p(s'|s_{\tau_l},a_{\tau_l}) \mathrm{d}s'}_{\geq 0} \leq \sum_{x \in I(s)} N_k(x,a) \int_{J} p(s'|x,a) \mathrm{d}s'
\end{align*}
Using Cauchy-Scwartz inequality:
\begin{align*}
 \sum_{J \in \mathcal{I}} \sqrt{V_k(J)} \leq \sqrt{S \sum_{J \in \mathcal{I}} V_k(J)} \leq \sqrt{S \sum_{x \in I(s)} N_k(x,a)  \sum_{J\in \mathcal{I}}\int_{J} p(s'|x,a) \mathrm{d}s'} = \sqrt{S N_k(I(s),a) }
\end{align*}
To conclude, we take a union bound over all possible $(I(s),a) \in \mathcal{I} \times \A$ and $J \in \mathcal{I}$. Note that we only need to take a union bound over $I(s) \in \mathcal{I}$ (and not $\calS$) because $s \mapsto \wt{p}_k(\cdot|s,a)$ and $s \mapsto \wt{r}_k(s,a)$ are piecewise constant on $\mathcal{I}$ (and similarly for $\wo{p}_k$ and $\wo{r}_k$). With probability at least $1-\frac{\delta}{15 t_k^6}$, for all $(s,a) \in \calS \times \A$ and for all $J \in \mathcal{I}$:
\begin{align}
\left| \wb{r}_k(s,a) - \wt{r}_k(s,a)  \right| \leq  \rmaxbound \sqrt{\frac{\lnn{\frac{90 S^2 A t_k^7}{\delta}}}{N_k(I(s),a)} } &\leq  \rmaxbound \sqrt{\frac{7\lnn{\frac{2 S A t_k}{\delta}}}{N_k(I(s),a)} } \label{eq:hp_bound_reward}\\
\sum_{J\in \mathcal{I}} \left|\int_{J}  \wb{p}_k(s'|s,a) - \wt{p}_k(s'|s,a)\mathrm{d}s' \right| &\leq  2 \sqrt{ \frac{7 S \lnn{\frac{3 S A t_k}{\delta}}}{N_k(I(s),a)}  } + \frac{28 S}{N_k(I(s),a)} \lnn{\frac{3 S At_k}{\delta}} \label{eq:hp_bound_proba}\\
 \left|\int_{\calS} \Big( \wb{p}_k(s'|s,a) - \wt{p}_k(s'|s,a) \Big)w^*(s') \mathrm{d}s' \right| &\leq  c \sqrt{\frac{7\lnn{\frac{2 S A t_k}{\delta}}}{N_k(I(s),a)}} \label{eq:hp_bound_proba_vector}
\end{align}
Since by definition $N_k^+(I(s),a):= \max \{ 1 , N_k(I(s),a)\}$, the above inequalities also hold with $N_k(I(s),a)$ replaced by $N_k^+(I(s),a)$.

\subsubsection{Gathering all the terms}

We first notice that $\int_{\calS} \left( \wh{p}_k(s'|s,a) - p(s'|s,a) \right) h^*(s')\mathrm{d}s' = \int_{\calS} \left( \wh{p}_k(s'|s,a) - p(s'|s,a) \right) w^*(s')\mathrm{d}s'$ since $w^*$ and $h^*$ are equal up to a constant shift and $\int_{\calS}\wh{p}_k(s'|s,a) \mathrm{d}s'= \int_{\calS}{p}(s'|s,a)\mathrm{d}s' = 1$. Gathering equations \eqref{eq:hp_bound_proba_vector}, \eqref{eq:pboundcontinuity2} and \eqref{eq:p.bound.perturbation.cont} we have:
\begin{align*}
 \left| \int_{\calS} \left( \wh{p}_k(s'|s,a) - p(s'|s,a) \right) h^*(s')\mathrm{d}s' \right| \leq c \sqrt{\frac{7\lnn{\frac{2 S A t_k}{\delta}}}{N_k^+(I(s),a)}} + cL S^{-\alpha} + \frac{c}{N_k(I(s),a) + 1}
\end{align*}
Gathering equations \eqref{eq:hp_bound_reward} and \eqref{eq:pboundcontinuity} we have:
\begin{align*}
 \left| \wb{r}_k(s,a) - r(s,a) \right| \leq \rmaxbound \sqrt{\frac{7\lnn{\frac{2 S A t_k}{\delta}}}{N_k^+(I(s),a)} } + \rmaxbound LS^{-\alpha} 
\end{align*}

Let $v$ be a piecewise constant function on $\mathcal{I}$ \st $\SP{v} \leq c$ and define $w(s) := v(s) - \left( \inf\{ v(s)\} + \sup\{ v(s)\} \right)/2$. $w$ is also piecewise constant on $\mathcal{I}$ and for all $J \in \mathcal{I}$, $w(J) \in [-c/2, c/2]$.  Gathering equations \eqref{eq:hp_bound_proba}, \eqref{eq:pboundcontinuity3} and \eqref{eq:p.bound.perturbation.cont2} we have that :
\begin{align*}
 \sum_{J \in \mathcal{I}}\left| \int_{J} \left( \wh{p}_k(s'|s,a) - p(s'|s,a) \right) v(s') \mathrm{d}s' \right| &= \sum_{J \in \mathcal{I}}\left| w(J) \int_{J} \left( \wh{p}_k(s'|s,a) - p(s'|s,a) \right) \mathrm{d}s' \right| \\
 & \leq \frac{c}{2} \sum_{J \in \mathcal{I}}\left| \int_{J} \left( \wh{p}_k(s'|s,a) - p(s'|s,a) \right) \mathrm{d}s' \right|\\
 & \leq c \sqrt{ \frac{7S \lnn{\frac{3 S A t_k}{\delta}}}{N_k^+(I(s),a)}  } + \frac{14 c S}{N_k^+(I(s),a)} \lnn{\frac{3 S At_k}{\delta}}  \\&~~~~+ cL S^{-\alpha}+ \frac{c}{N_k(I(s),a) + 1}
\end{align*}
Properties \emph{(a)} and \emph{(b)} of Lem.~\ref{L:bonus.inequality} follow by definition of the exploration bonuses and application of the triangle inequality.

\subsection{Optimism (Proof of Lem.~\ref{L:optimism})}

Let $\wh{g}^{ag+}_k$ denote the solution of optimisation problem \eqref{eq:opt.superior.spanc} on $\wh{M}^{ag+}_k$. In this section we prove that:
\begin{lemma}\label{L:optimism.cont}
        Consider the MDP $\wh{M}^{ag+}_k$ defined in Sec.~\ref{S:exploration.bonus.continuous}. Then for any $k>0$, with probability at least $1-\frac{\delta}{15 t_k^6}$, $\wh{g}^{ag+}_k \geq g^*$.
\end{lemma}
$\wh{M}^{ag+}_k$ only has a finite number of states while the true MDP $M^*$ has an uncountable state-space. Thus, it is difficult to compare directly $\wh{g}^{ag+}_k$ with $g^*$. To overcome this difficulty, we first compare $g^*$ with the gain of $\wh{M}_k$ and then compare the latter to $\wh{g}^{ag+}_k$.

\paragraph{1. Optimism of $\wh{M}_k$.}
Let $\wh{g}_k$ denote the solution of optimisation problem \eqref{eq:opt.superior.spanc} on $\wh{M}_k$. 
To prove that $\wh{g}_k \geq g^*$ we can use Prop.~\ref{prop:bound.gain} which --as explained in the main body of the paper-- only requires to show that $\wh{L}h^* \geq Lh^*$ where $\wh{L}$ is the optimal Bellman operator of $\wh{M}_k$. By applying property \emph{(a)} of Lem.~\ref{L:bonus.inequality.continuous}, we have that with probability at least $1-\frac{\delta}{15 t_k^6}$:
\begin{align*}
    \forall s\in \calS,~\wh{L} h^*(s) &:= \max_{a \in \A} \left\{ \wb{r}_k(s,a) + b_k(s,a) + \int_{\calS} \wh{p}_k(s'|s,a) h^*(s') \mathrm{d}s' \right\}\\
                 &\geq \max_{a\in\A} \left\{ r(s,a) + \int_{\calS} p(s'|s,a) h^*(s') \mathrm{d}s'\right\} = Lh^*(s) 
\end{align*}
Therefore, $\wh{g}_k \geq g^*$ with probability at least $1-\frac{\delta}{15 t_k^6}$.

\paragraph{2. Relationship between $\wh{M}_k$ and $\wh{M}^{ag+}_k$.} 
We now show that $\wh{g}^{ag+}_k = \wh{g}_k$.
Consider a piecewise-constant function $v_0$ on $\mathcal{I}$ (\eg $v_0 = 0$) and a vector $u_0 \in \Re^S$ satisfying $u_0(J) = v_0(J)$ for all $J \in \mathcal{I}$. We define the sequences $v_{n+1} := \wh{T}_c v_n$ and $u_{n+1} := \wh{T}_c^{ag+} u_n$. We show by induction that $u_n(J) = v_n(J)$ for all $n \geq 0$ and for all $J \in \mathcal{I}$. By definition it is true for $n=0$ and for all $n \geq 0$:
\begin{align}\label{eq:piecewise_vector_proba}
\begin{split}
   \int_{s \in \calS} \wh{p}_k(s'|s,a) v_n(s') \mathrm{d}s' 
       &= \sum_{J \in \mathcal{I}} \int_{J} \wh{p}_k(s'|s,a) v_n(s') \mathrm{d}s'\\
       &= \sum_{J \in \mathcal{I}} v_n(J) \int_{J} \wh{p}_k(s'|s,a)\mathrm{d}s' = \sum_{J \in \mathcal{I}} u_{n}(J)\; \wh{p}^{ag}_k(J| I(s),a)
\end{split}
\end{align}
where the last equality follows from \eqref{eq:integral.p.perturbed} and the induction hypothesis.
In addition $\wh{r}_k(s,a)$ is also piecewise-constant on $\mathcal{I}$ and $\wh{r}_k(s,a) = \wh{r}^{ag}_k(I(s),a)$ for all $s \in \calS$. Therefore, we have that $\wh{L}^{ag} u_n(I(s)) = \wh{L} v_n(s)$ for any $s \in \calS$. Finally, the augmentation is not impacting the optimal Bellman operator (\ie for any $v$, $\wh{L}^{ag+} v = \wh{L}^{ag}v$) so $\wh{L}^{ag+} u_n(I(s)) = \wh{L} v_n(s)$ and consequently $\wh{T}^{ag+}_c u_n(I(s)) = \wh{T}_c v_n(s)$ for any $s \in \calS$. This shows that $v_{n+1}(J) =u_{n+1}(J)$ for all $J \in \mathcal{I}$ which concludes the proof by induction.

As shown by \citet[Theorem 10]{fruit2018constrained}, $ \lim_{n \to +\infty}v_{n+1}(J) - v_n(J) = \wh{g}^{ag+}_k$ and $ \lim_{n \to +\infty}u_{n+1}(J) - u_n(J) = \wh{g}_k$ so that $\wh{g}^{ag+}_k = \wh{g}_k \geq g^*$ with probability at least $1-\frac{\delta}{15 t_k^6}$.
%

%% file: app_proof_regret_cont.tex
\subsection{Regret Proof of \scalcont (Proof of Thm.~\ref{thm:regret.scalcont})} \label{app:proof.regret.continuous} 
In this section, we provide a complete proof of the regret bound for \scalcont.
Defining $\Delta_k=\sum_{s \in \mathcal{S}} \nu_k(s) \left(g^* - \sum_{a \in \A_{s_t}} r(s,a)\widetilde{\pi}_{k}(s,a)\right)$ and using the arguments in~\citep{Jaksch10,fruit2018constrained}, it holds with probability at least $1-\frac{\delta}{20T^{5/4}}$ that:
 $\Delta(\bonusscal,T) \leq \sum_{k=1}^m \Delta_k + \rmaxbound\sqrt{\frac{5}{2}T\ln \left(\frac{11T}{\delta}\right)}$.
Note that $\nu_k(s)$ is the total number of observation of state $s$ in episode $k$ and is well-defined for $s$ lying in $[0,1]$.
Finally, recall that for any subset $I\subseteq \mathcal{S}$, the sum $\sum_{s \in I} u_{s} $ is also well-defined as long as the collection $\left(u_{s} \right)_{s \in I}$ contains only a finite number of non-zero elements.

\vspace{.2cm}
\noindent \tikz{
        \node[draw=red!90!black, color=red!90!black, text width=0.985\textwidth]{\faExclamationTriangle~~~ In this section we will abuse of notation and write $p(\cdot|s,a)^\transp v = \int_\calS p(s'|s,a) v(s') \mathrm{d}s'$ for any probability density function $p$ defined on $\calS = [0,1]$.};
}

\subsubsection{Optimism and Bellman Equation}
We consider the case where $\Delta_k \cdot \one{g^*_c(\wh{M}^{ag+}_k) \geq g^*}$, the complementary case is (cumulatively) bounded by $\rmaxbound \sqrt{T}$ as in~\citep{Jaksch10,fruit2018constrained}. 
Denote by $g_k:= 1/2(\max \{\wh{T}^+_c v_{k} - v_k\} + \min \{\wh{T}^+_c v_{k} - v_k\})$ where $v_k$ is the value function returned by \scopt{($0,\overline{s},\gamma_k,\varepsilon_k$)}. 
Remember that $v_k \in \mathbb{R}^S$ is a discrete vectors obtained by applying \scopt on $\wh{M}_k^{ag+}$.
The stopping condition of \scopt is such that~\citep[see][]{fruit2018constrained}
\begin{align*}
        g_k \geq g_c^*(\wh{M}_k^{ag+}) - \underbrace{\varepsilon_k}_{=\rmaxbound /\sqrt{t_k}} \overbrace{\geq}^{\text{Lem.~\ref{L:optimism}}} g^* - \frac{\rmaxbound}{\sqrt{t_k}}
\end{align*}
implying:
\begin{align*}
        \Delta_k \cdot \one{g^*_c(\wh{M}_k^{ag+}) \geq g^*} 
\leq \rmaxbound \sum_{s \in \mathcal{S}} \frac{ \nu_k(s)}{\sqrt{t_k}} 
+ \sum_{s \in \mathcal{S}} \nu_k(s) \Bigg(
\underbrace{g_k - \sum_{a \in\mathcal{A}_{s}} r(s,a)\widetilde{\pi}_{k}(s,a)}_{:= \Delta'_k(s)}
\Bigg)
\end{align*}
Note that we can associated a continuos piecewise-constant function to the discrete vector $v_k$: $u_k(s) = v_k(I(s)),~\forall s\in \calS$.
A consequence of Lem.~\ref{L:bonus.inequality.continuous}\emph{(b)} applied to vector $u_k$ is that $-r(s,a) \leq d_k(s,a) -\wb{r}_k(s,a) + \left( p(\cdot|s,a) -  \wh{p}_k(\cdot|s,a) \right)^\transp u_k$. Note that we cannot use the tighter version $\emph{(a)}$ since it does not hold for any function $u_k$. Moreover, by definition $\wb{r}_k(s,a) = \wb{r}_k^{ag}(I(s),a) = \wh{r}_k(s,a) - b_k(s,a)$. Therefore:
\begin{equation}
        \label{E:delta_prime}
        \begin{aligned}
        \Delta'_k(s) &\leq 
        g_k -  \sum_{a \in \A_s} {\pi}_k(s,a)\bigg( \underbrace{\wh{r}_k(s,a)}_{:= \wh{r}^{ag}_k(I(s),a)}  + \wh{p}_k(\cdot|s,a)^\transp u_k \bigg)\\
        &\quad{} + \sum_{a \in \A_s} {\pi}_k(s,a) \bigg( \underbrace{b_k(s,a)}_{:= b_k(I(s),a)} + d_k(s,a) + p(\cdot|s,a) ^\transp u_k \bigg)
        \end{aligned}
\end{equation}
A direct consequence of the stopping condition used by \scopt (see Thm. 18 in~\citep{fruit2018constrained}) is that: $\forall J \in \mathcal{I}$,
\begin{equation}\label{E:gerror.app2}
        \bigg| g_k - \sum_{a \in \A} \sum_{i \in \{1,2\}} \wh{r}_k^{ag+}(J, a_i) {\pi}_k(J,a_i) + v_k(J)
        - \sum_{a \in \A} \sum_{i \in \{1,2\}} {\pi}_k(J,a_i)\wh{p}_k^{ag+}(\cdot|J,a_i)^\transp v_k \bigg| \leq \frac{\rmaxbound }{\sqrt{t_k}}
\end{equation}
Recall that by definition: ${\pi}_k(J,a) = {\pi}_k(J,a_1) + {\pi}_k(J,a_2)$, $\wh{r}_k^{ag+}(J, a_i) \leq \wh{r}_k^{ag}(J, a)$ (since $\wh{r}_k^{ag+}(J, a_1) = \wh{r}_k^{ag}(J, a)$ and $\wh{r}_k^{ag+}(J, a_2) = 0$) and $\wh{p}_k^{ag}(\cdot|J,a) = \wh{p}_k^{ag+}(\cdot|J,a_i) $. We can thus write:
\begin{equation}\label{eq:mdp.augmentation.props}
\begin{aligned}
        &\sum_{a \in \A} \wh{r}_k^{ag}(J,a) {\pi}_k(J,a) 
        = \sum_{a \in \A} \sum_{i \in \{1,2\}} \wh{r}_k^{ag}(J, a) {\pi}_k(J,a_i)
        \geq\sum_{a \in \A} \sum_{i \in \{1,2\}} \wh{r}_k^{ag+}(J, a_i) {\pi}_k(J,a_i)\\
        & \text{and} ~ \sum_{a \in \A} {\pi}_k(s,a)  \wh{p}_k^{ag}(\cdot|J,a) 
        = \sum_{a \in \A} \sum_{i \in \{1,2\}} {\pi}_k(J,a_i) \wh{p}_k^{ag+}(\cdot|J,a_i) 
\end{aligned}
\end{equation}
Plugging this last two equations into \eqref{E:gerror.app2} and
%
using Eq.~\ref{eq:piecewise_vector_proba} (\ie $\wh{p}_k(\cdot|s,a)^\transp u_k = \wh{p}^{ag}_k(\cdot|s,a)^\transp v_k$), the fact $u_k(s) = v_k(I(s))$, $\wh{r}_k(s,a) = \wh{r}^{ag}_k(I(s),a)$, and ${\pi}_k(s,a) = {\pi}_k(I(s),a)$, we obtain:
\begin{equation}\label{E:gerror.app}
        \begin{aligned}
                \forall s \in \calS, ~~~ g_k - \sum_{a \in \A} {\pi}_k(s,a) \bigg( \wh{r}_k(s,a)  + \wh{p}_k(\cdot|s,a)^\transp u_k\bigg) \leq - u_k(s) + \frac{\rmaxbound }{\sqrt{t_k}}
\end{aligned}
\end{equation}
Combining \eqref{E:gerror.app} with \eqref{E:delta_prime} we have:
\begin{align*}
        \Delta'_k(s) \leq \sum_{a\in\A_s} {\pi}_k(s,a) \bigg( d_k(s,a) + \underbrace{b_k(s,a)}_{\leq d_k(s,a)} + p(\cdot|s,a)^\transp u_k \bigg) - u_k(s) + \frac{\rmaxbound }{\sqrt{t_k}}\\
\end{align*}
Note that $d_k(s,a) \geq b_k(s,a)$ for any $(s,a) \in \calS \times \A$ since the term $\phi_{p,k}^{Ia}$ (see Eq.~\ref{eq:bernstein.ci.cont}) contains a $\sqrt{S}$ dependence that is not present in $\beta_{p,k}^{Ia}$.
Since the dominant term is given by $d_k(s,a)$, we will consider the following loser upper-bound $d_k(s,a) + b_k(s,a) \leq 2 d_k(s,a)$ in the remaining of the proof.
We can now state that
\begin{equation}
        \label{eq:deltak.bound}
\begin{aligned}
        \Delta_k &\leq  \underbrace{\sum_s \nu_k(s) \left( \sum_a \left( {\pi}_k(s,a) p(\cdot|s,a)^\transp w_k \right) - w_k(s) \right)}_{\xi_k}\\
                 &\quad{} + 2 \sum_{s,a} \nu_k(s) {\pi}_k(s,a)  d_k(s,a) + 2\rmaxbound \sum_{s \in \mathcal{S}} \frac{ \nu_k(s)}{\sqrt{t_k}}
\end{aligned}
\end{equation}
where $w_k = u_k - (\inf_s\{u_k(s)\} + \sup_s \{u_k(s)\})/2$ is obtained by ``recentering'' $u_k$ around $0$ so that $\|w_k\|_\infty = \SP{w_k}/2 \leq c/2$~\citep[see][App. F4]{fruit2018constrained}.
Then, similarly to what is done in~\citep[][Sec. 4.3.2]{Jaksch10} and~\citep[][App. F.7, pg. 32]{fruit2018constrained}, we have
\begin{align*}
        \xi_k 
        &= \sum_{t=t_k}^{t_{k+1}-1} 
        \underbrace{\sum_a \int p(s'|s_t,a_t) {\pi}_k(s_t,a_t) w_k(s') \mathrm{d}s'  - w_k(s_{t+1})}_{:=X_t} 
        +  \sum_{t=t_k}^{t_{k+1}-1} w_k(s_{t+1}) - w_k(s_t)\\
        &= \sum_{t=t_k}^{t_{k+1}-1} X_t
        +  \underbrace{w_k(s_{t_k+1}) - w_k(s_{t_k})}_{\leq \SP{w_k} \leq c}
\end{align*}
Given the filtration $\mathcal{F}_t = \sigma\left(s_1,a_1,r_1,\dots, s_{t+1} \right)$, $X_t$ is an MDS since $|X_t| \leq c$ and $\mathbb{E}[X_t|\mathcal{F}_{t-1}]=0$ since ${\pi}_{k_t}$ is $\mathcal{F}_{t-1}$-measurable.
By using Azuma inequality we have that with probability at least $1 - \frac{\delta}{20 T^{5/4}}$
\begin{equation}\label{eq:pminusone.inequality}
        \sum_{k=1}^m \xi_k \leq c \sqrt{\frac{5}{2}T\ln \left(\frac{11T}{\delta}\right)} +   m c
\end{equation}
with $m \leq SA \log_2\left(\frac{8T}{SA}\right)$ when $T \geq SA$ (see App. C.2 in~\citep{Jaksch10}).

\subsubsection{Bounding the exploration bonus}
Using the same argument in App. F.6 in~\citep{fruit2018constrained} and by noticing that $d_k(s,a) \leq 2c + \rmaxbound \leq 2\max\{c, \rmaxbound\}$, we obtain with probability at least $1-\frac{\delta}{20T^{5/4}}$:
\begin{align}
        \sum_{k=1}^{m}\sum_{s,a} \nu_k(s)&{\pi}_k(s,a) d_k(s,a)
        \leq \sum_{k=1}^{m}\sum_{s,a} \nu_k(s,a) d_{k}(s,a)
        + 2\max\{c,\rmaxbound\}\sqrt{\frac{5}{2}T\ln \left(\frac{11T}{\delta}\right)}
        \label{eq:bonus.mds}
\end{align}

We now gather inequalities~\eqref{eq:bonus.mds},~\eqref{eq:pminusone.inequality} and a result in~\citep[][App. F.7, pg. 33]{fruit2018constrained} into inequality~\eqref{eq:deltak.bound} summed over all the episodes $k$ which yields (after taking a union bound) that with probability at least $1 - \frac{2\delta}{20 T^{5/4}}\geq 1 - \frac{3\delta}{20 T^{5/4}}$ (for $T\geq SA$):
\begin{align}\label{eq:sum.deltak}
 \begin{split}
         \sum_{k=1}^m \Delta_k \;\mathbbm{1}\{g^*_c(\wh{M}_k) \geq g^*\} &\leq
         3\max\{c, \rmaxbound\} \sqrt{\frac{5}{2}T\ln \left(\frac{11T}{\delta}\right)}
         +  c SA\log_2\left(\frac{8T}{SA}\right)\\
 &\quad{}
 + 2\rmaxbound \left(\sqrt{2} +1 \right)\sqrt{SAT}
 + 2
        \sum_{k=1}^{m}\sum_{s,a} \nu_k(s,a) d_{k}(s,a)
 \end{split}
\end{align}
Let $\phi_{p,k}^{sa}$ as defined in Eq.~\ref{eq:bernstein.ci.cont}, then 
\begin{equation}\label{eq:bonus.bound.2}
\begin{aligned}
        \sum_{k=1}^{m}\sum_{s,a} \nu_k(s,a) d_{k}(s,a) 
        &\leq
        \underbrace{\sum_{k=1}^{m}\sum_{s,a} \nu_k(s,a) \beta_{r,k}^{sa}}_{\text{see p. 33~\citep{fruit2018constrained}}}
        +c \underbrace{\sum_{k=1}^m \sum_{s,a} \nu_k(s,a) \phi_{p,k}^{sa}}_{\text{see p. 33~\citep{fruit2018constrained}}}\\
        &\quad{}+ 2c \sum_{k=1}^m \sum_{s,a}\frac{\nu_k(s,a)}{N_k(s,a)+1} 
        + \underbrace{(c+\rmaxbound)}_{\leq 2\max\{c,\rmaxbound\}} L S^{-\alpha} T
\end{aligned}
\end{equation}
We recall that 
\begin{align*}
        \sum_{k=1}^{m}\sum_{s,a} \nu_k(s,a) \beta_{r,k}^{sa} = \wt{O}(\rmaxbound \sqrt{SAT}) ~~ \text{ and } ~~ \sum_{k=1}^m \sum_{s,a} \nu_k(s,a) \phi_{p,k}^{sa} =\wt{O}(c S\sqrt{A T} +cS^2A).
\end{align*}
Similarly to what done in~\citep[][Eq. 58-60]{fruit2018constrained}, we can write
\begin{equation}\label{eq:nu.over.n}
        \begin{split}
        \sum_{k=1}^m &\sum_{s,a}\frac{\nu_k(s,a)}{N_k(s,a)+1}
\leq 2\sum_{s,a} \sum_{t=1}^{T} \frac{\mathbbm{1}_{\lbrace (s_t, a_t) =(s,a)\rbrace}}{N_k(s_t,a_t)+1}
= 2 \sum_{s,a} \sum_{j=1}^{N_{T+1}(s,a)} \frac{1}{j+1} \leq 2 SA \ln(T)
\end{split}
\end{equation}

\subsubsection{Completing the proof}\label{sec:regret.summingup}
Summing up all the contributions and taking a union bound  over all possible values of $T$ and use the fact that $\sum_{T= 2}^{+\infty}\frac{\delta}{4T^{5/4}}<\delta$, we write that there exists a numerical constant $\chi$ such that at least with probability $1-\delta$ our algorithm \scalcont has a regret bounded by
\begin{align*}
        \begin{split}
                \Delta(\scalcont,T) \leq &\chi \;\left( \max{\left\lbrace \rmaxbound,c \right\rbrace} \left( S\sqrt{A T \ln\left(\frac{T}{\delta}\right)}
                        + S^2A\ln^2\left(\frac{T}{\delta} \right) 
                                          + L S^{-\alpha}T \right)
                \right)
         \end{split}
\end{align*}

We now set $S = \left(\alpha L \sqrt{\frac{T}{A}}\right)^{\sfrac{1}{(\alpha + 1)}}$ so that
\begin{align*}
    \Delta(\scalcont, T) 
    &= \wt{O} \bigg( \max\{\rmaxbound,c\} \bigg( \underbrace{\max \left\{ \alpha^{\sfrac{1}{(\alpha+1)}}, \alpha^{-\sfrac{\alpha}{(1+\alpha)}} \right \}}_{\leq 2,~\forall \alpha \geq 0} \times\\
    & \qquad \quad \times L^{\sfrac{1}{(\alpha+1)}} A^{\sfrac{\alpha}{(2\alpha+2)}} T^{\sfrac{(\alpha+2)}{(2\alpha+2)}} + \underbrace{\alpha^{\sfrac{2}{(1+\alpha)}}}_{\leq 2,~\forall \alpha >0} L^{\sfrac{2}{(1+\alpha)}} A^{\sfrac{\alpha}{(1+\alpha)}} T^{\sfrac{1}{(1+\alpha)}} \bigg ) \bigg)\\
\end{align*}
Finally, when $T \geq L^{\sfrac{2}{\alpha}}A$, the regret of \scalcont is bounded with probability at least $1-\delta$ by
\[
        \Delta(\scalcont,T) = \wt{O}\left(
            \max\{\rmaxbound,c\}L^{\sfrac{1}{(\alpha+1)}} A^{\sfrac{\alpha}{(2\alpha+2)}} T^{\sfrac{(\alpha+2)}{(2\alpha+2)}}
        \right).
\]

%% file: app_regret_finite.tex
\section{Finite MDPs: the analysis of \bonusscal}\label{app:finite}
In this section we analyse \bonusscal by leveraging the results provided for the continuous state case.
We start presenting the derivation of the bonus $b_k$ and an analogous of Lem.~\ref{L:bonus.inequality.continuous} which implies \bonusscal is optimistic at each episode $k$.
Finally, we provide the proof of the regret bound stated  in Thm.~\ref{thm:regret.scalcont}.

\subsection{High probability bound using the exploration bonus (proof of Lem.~\ref{lem:explbonus.inequality})}\label{app:proba_bound_bonus.discrete}
To begin with, we introduce two variants of the exploration bonus that will be used for the regret proof:
\begin{equation}\label{eq:bonus.proof.finite}
\begin{aligned}
        b_k(s,a) &:= c \cdot \min \left\{ \beta_{p,k}^{sa} + \frac{1}{N_k(s,a) + 1}; 2\right\} + \min \left\{ \beta_{r,k}^{sa}; \rmaxbound\right\}\\
        d_k(s,a) &:= c \cdot \min \left\{ \phi_{p,k}^{sa} + \frac{1}{N_k(s,a) + 1}; 2\right\} + \min \left\{ \beta_{r,k}^{sa}; \rmaxbound\right\}
\end{aligned}
\end{equation}
where $\beta_{p,k}^{sa} = \beta_k^{sa}$ (see Eq.~\ref{eq:exploration.bonus}), $\beta_{r,k}^{sa} = \rmaxbound \beta_k^{sa}$ and
\begin{align*}
        \phi_{p,k}^{sa} := \sqrt{\frac{7 (\Gamma -1)\lnn{\frac{3 SA t_k}{\delta}}}{\max\{1,N_k(s,a)\}}} + \frac{14S}{\max\{1,N_k(s,a)\}} \lnn{\frac{3SAt_k}{\delta}}
\end{align*}
Notice that compared to the bonus $b_k$, $d_k$ explicitly depends on the number of states (linearly in $S$) and next states (sublinearly in $\Gamma$).
As a consequence, $d_k(s,a) \geq b_k(s,a)$ for any $(s,a) \in \calS \times \A$.
In the continuous case we might consider the number of next states in the (true) aggregated MDP. However, this quantity is not very informative so we have decided (for sake of clarity) to upper-bound it by the number of intervals.

\begin{lemma}\label{L:bonus.inequality}
        Let $(g^*,h^*)$ be a solution of the optimality equation $Lh^* = h^* + g^*$ such that $\SP{h^*}  \leq c$. For all $T \geq 1$ and $k \geq 1$, with probability at least $1-\frac{\delta}{15 t_k^6}$, for any $(s,a) \in \calS \times \A$ and for any $v \in \mathbb{R}^S$ \st $\SP{v} \leq c$ we have:
        \begin{align*}
                &(a) ~~ b_{k}(s,a) \geq \left| \wb{r}_k(s,a) - r(s,a) + \left( \wh{p}_k(\cdot|s,a) - p(\cdot|s,a) \right)^\transp h^* \right|\\
                &(b) ~~ d_k(s,a) \geq \left| \wb{r}_k(s,a) - r(s,a) + \left( \wh{p}_k(\cdot|s,a) - p(\cdot|s,a) \right)^\transp v \right|
        \end{align*}
        where $b_k$ and $d_k$ are defined as in Eq.~\ref{eq:bonus.proof.finite}.
\end{lemma}
\begin{proof}
        We consider the discrete case as a special sub-case of the continuous one considered in Lem.~\ref{L:bonus.inequality.continuous}.
        As explained in Sec.~\ref{app:bound_diff_tilde_bar}, for the discrete case we can even use an independence argument based on ``stack of samples'' idea as done for bandits~\citep[][Sec. 4.4]{bandittorcsaba}. However, for sake of clarity we use the same MDS argument exploited in the continuous case.
        The main difference is that in the discrete case we do not need state aggregation and thus we replace every interval with a singleton function, \ie $I(s) = s,~\forall s \in \calS$.
        Define $w := h^* - (\min\{h^*\} + \max\{h^*\})/2$ such that $w \in [-c/2, c/2]$.
        We decompose $\wh{p}_k -p$ into $(\wh{p}_k - \wb{p}_k) + (\wb{p}_k - p)$. As done in Eq.~\ref{eq:p.bound.perturbation.cont} (Sec.~\ref{app:bound_diff_hat_bar}), we can write that
        \begin{equation}\label{eq:bound.perturbation.p}
        \begin{aligned}
                \big|  (\wh{p}_k(\cdot|s,a) - \wb{p}_k(\cdot|s,a) )^\transp w \big| 
                &\leq \left|\frac{N_k(s,a)}{N_k(s,a)+1} -1 \right| \underbrace{\|\wb{p}_k(\cdot|s,a)\|_1}_{=1} \underbrace{\|w\|_{\infty}}_{\leq c/2} +\frac{|w(\wb{s})|}{N_k(s,a)+1} \\
                &\leq \frac{c}{2} \left( 1 - \frac{N_k(s,a)}{N_k(s,a)+1} + \frac{1}{N_k(s,a)+1} \right)
                = \frac{c}{N_k(s,a) + 1} 
        \end{aligned}
        \end{equation}
        In order to bound the term depending on $(\wb{p}_k - p)$ we use the same MDS argument as in Sec.~\ref{app:bound_diff_tilde_bar}.
        You can consider $p$ equivalent to $\wt{p}$ defined in the continuous case since:
        \[
                \wt{r}_k(s,a) = \frac{1}{N_k(\underbrace{I(s)}_{:=s},a)} \sum_{x \in \underbrace{I(s)}_{:=s}} N_k(x,a) r(x,a) = r(s,a).
        \]
        Similarly, we can prove that $\wt{p}_k(s'|s,a) = p(s'|s,a)$. Then, we consider the same adapted sequences, stopping times and predictable processes except from the fact that intervals are replaced by singletons (\ie discrete states).
        As a consequence, (an analogous of) Lem.~\ref{L:conditioning_stop_time} holds.
        By following the same steps in Sec.~\ref{app:bound_diff_tilde_bar}, we can prove that with probability at least $1-\delta$, for all $(s,a) \in \calS \times \A$
        \begin{align*}
                \left| \wb{r}_k(s,a) - r_k(s,a)  \right| &\leq  \rmaxbound \sqrt{\frac{7\lnn{\frac{2 S A t_k}{\delta}}}{N_k^+(s,a)} } := \beta_{r,k}^{sa} \\
                \left|\big( \wb{p}_k(\cdot|s,a) - p(\cdot|s,a) \big)^\transp h^* \right| &\leq  c \sqrt{\frac{7\lnn{\frac{2 S A t_k}{\delta}}}{N_k^+(s,a)}} := c\beta_{p,k}^{sa}
        \end{align*}
        where we recall the $N_k^+(s,a) := \max\{1, N_k(s,a)\}$.
        We now consider the concentration of $(\wh{p}_k -p)^\transp v$ for which we need to use Freedman's inequality (see Thm.~\ref{T:inequality_stop_mds}).
        Similarly to what done before, let $z = v - (\min\{v\} + \max\{v\})/2$ such that $(\wh{p}_k - p)^\transp v = (\wh{p}_k - p)^\transp z$. 
        We start noticing that, Eq.~\ref{eq:bernstein_transition_proba} holds for the discrete case where we replace the adapted sequence $\one{s_{\tau_l+1} \in I}$ by $\one{s_{\tau_l+1} = s'}$ 
        and the conditional variance $V_k(J)$ by $V_k(s') = \sum_{l=1}^{N_k(s,a)} \one{\tau_l < t_k} \Varcond{\one{s_{\tau_l+1} =s'}}{\G_{l-1}}$.
        Furthermore, $\Varcond{\one{s_{\tau_l+1} = s'}}{\mathcal{G}_{l-1}} = (1-p(s'|s,a)) p(s'|s,a)$ and
        \[
                V_k(s') = \sum_{l=1}^{N_k(s,a)} (1 - p(s'|s_{\tau_l}, a_{\tau_l})) p(s'|s_{\tau_l},a_{\tau_l}) \leq N_k(s,a) (1 - p(s'|s,a)) p(s'|s,a) 
        \]
        As done in~\citep[][App. F.7]{fruit2018constrained} we use Cauchy-Schwartz inequality to write that
        \begin{align*}
                \sum_{s' \in \mathcal{S}} \sqrt{p(s'|s,a)(1-p(s'|s,a))} &= \sum_{s' \in \mathcal{S}:~ p(s'|s,a)>0} \sqrt{p(s'|s,a)(1-p(s'|s,a))} \leq \sqrt{\Gamma -1}
        \end{align*}
        where we recall that $\Gamma := \max_{s,a} \| p(\cdot|s,a) \|_0$ is the maximum support of $p$.
        Then, for any $(s,a) \in \calS \times\A$ and for any vector $z \in [-c/2,c/2]$, we have that
        \begin{align*}
                \left|\big(\wb{p}_k(\cdot|s,a) - p(\cdot|s,a) \big)^\transp z\right| 
                &\leq \|z\|_{\infty}\sum_{s' \in \calS} |\wb{p}_k(s'|s,a) - p(s'|s,a)| \\
                &\leq \frac{c}{2 N_k^+(s,a)} \sum_{s' \in \calS} \left(2 \sqrt{V_{k}(s') \lnn{\frac{4 t_k}{\delta}}} + 4 \lnn{\frac{4t_k}{\delta}} \right)\\
                &\leq c  \left( \sqrt{\frac{7 (\Gamma -1)\lnn{\frac{3 SA t_k}{\delta}}}{N_k^+(s,a)}} + \frac{14S}{N_k^+(s,a)} \lnn{\frac{3SAt_k}{\delta}} \right)
                := c\; \phi_{p,k}^{sa}
        \end{align*}
        We can also write with probability $1$ that:
        \begin{align*}
                \big| (\wh{p}_k(\cdot|s,a) - p(\cdot|s,a) )^\transp w \big|
                \leq \wh{p}_k(\cdot|s,a) ^\transp w +  p(\cdot|s,a)^\transp w 
                 \leq 2 c
        \end{align*}
        So we can take the minimum between the two upper-bounds. We also know that the difference in reward is bound by $\rmaxbound$.
\end{proof}

In order to prove optimism we start noticing that the bonus $b_k(s,a)$ (see Lem.~\ref{L:bonus.inequality}) implies that $\wh{L}_k h^* \geq Lh^*$. As a consequence, we can use Prop.~\ref{prop:bound.gain} to show that $g^+_k \geq g^*$.

\subsection{Regret Proof of \bonusscal}
The regret proof follows the same steps of the one for \scalcont.
The main difference resides in the fact that there is no need of state aggregation, thus simplifying the proof.

By using the optimism of $\wh{M}^+_k$,  the stopping condition of \scopt and the relationships between $\wh{M}^+_k$ and $\wh{M}_k$ (see Eq.~\ref{eq:mdp.augmentation.props}), we can prove Eq.~\ref{eq:deltak.bound} for the discrete case.
Note that the analysis of the cumulative contribution of the term $d_k(s,a)$ and $b_k(s,a)$ will lead to the following terms $\wt{O}(c\sqrt{\Gamma SAT})$ and $\wt{O}(c \sqrt{SAT})$, respectively. Since the dominant term is the one associated to $d_k$, even in this case we upper-bound $b_k$ by $d_k$.

From this point, we follow the same steps as in Sec.~\ref{app:proof.regret.continuous}. The only difference resides in Eq.~\ref{eq:sum.deltak} where the term $(c + \rmaxbound) LS^{-\alpha} T$ disappears since it depends on aggregation and/or smoothness.
Finally, the regret bound in Thm.~\ref{thm:regret.bonusscal} follows by noticing that the order of the term $\sum_{k=1}^m \sum_{s,a} \nu_k(s,a) \phi_{p,k}^{sa}$ is $\wt{O}(\sqrt{\Gamma S A T} + S^2 A)$.

As a consequence, there exists a numerical constant $\chi$ such that at least with probability $1-\delta$ our algorithm \bonusscal has a regret bounded by
\begin{align*}
        \begin{split}
                \Delta(\bonusscal,T) \leq &\chi \;\left( \max{\left\lbrace \rmaxbound,c \right\rbrace} \left( \sqrt{\Gamma S A T \ln\left(\frac{T}{\delta}\right)}
                        + S^2A\ln^2\left(\frac{T}{\delta} \right) 
                                          \right)
                \right)
         \end{split}
\end{align*}

%% file: app_probability.tex

\section{Results of probability theory} \label{app:prob.res}

\subsection{Reminder}

We start by recalling some well-known properties of filtrations, stopping times and martingales \citep[Chapter 2]{klenke2013probability}. For simplicity, we use ``\as'' to denote ``almost surely'' (\ie~with probability 1). 
In this section, we consider a probability space $(\Omega, \F, \Prob)$. We call \emph{filtration} any \emph{increasing} (for the inclusion) sequence of sub-\sigalgs~ of $\F$ \ie $(\F_n)_{n \in \Na}$ where $\forall n \in \Na$, $ \F_{n} \subseteq \F_{n+1} \subseteq \F$. We denote by $\F_{\infty} := \cup_{n \in \Na} \F_n$. For any sub-\sigalg~ $\G \subseteq \F$, we say that a real-valued random variable (\rv) $X : \Omega \rightarrow \Re^d$ is $\G$-measurable if for all borel sets $B \in \mathcal{B}\left(\Re^d\right)$, $X^{-1}(B) \in \G$. We say that $X$ is \emph{$\G$-integrable} if it is $\G$-measurable and $\E \big[|X |\big] < +\infty$ (componentwise). We call \emph{stochastic process} any sequence of \rv $ (X_n)_{n \in \Na}$. We say that the stochastic process $(X_n)_{n \in \Na}$ is \emph{adapted} to the filtration $(\F_n)_{n \in \Na}$ if for all $n\in \Na$, $X_n$ is $\F_n$-measurable. In this case, the sequence $(X_n, \F_n)_{n \in \Na}$ is called an \emph{adapted sequence}. If in addition, $X_n$ is integrable for all $n \in \Na$ then we say that $(X_n, \F_n)_{n \in \Na}$ is an integrable adapted sequence. We say that a stochastic process $(X_n)_{n \in \Na}$ is almost surely:
\begin{enumerate}[topsep=0pt, partopsep=0pt]
 \itemsep0em 
 \item \emph{increasing} (resp. \emph{strictly increasing}) if for all $n \geq N$, $\Prob \left( X_n \leq X_{n+1}  \right) =1$ (resp. $\Prob \left( X_n < X_{n+1}  \right) =1$),
 \item \emph{bounded} if there exists a universal constant $K$ such that for all $n \in \Na$, $\Prob \left( X_n < K \right) =1$,
\end{enumerate}

\begin{definition}[Conditional expectation]\label{D:cond_exp}
 Let $X$ be an $\F$-integrable \rv with values in $\Re^d$. Let $\G \subseteq \F$ be a sub-\sigalg of $\F$. The \emph{conditional expectation of $X$ given $\G$} (denoted $\E \big[X \big| \G \big]$) is the (\as~unique) \rv that is $\G$-integrable and satisfies:
 \[ \forall A \in \G, ~~ \E \big[\one{A} \cdot\E \big[X \big| \G \big]\big] = E \big[\one{A} \cdot X \big] \]
\end{definition}

\begin{proposition}[Law of total expextations]\label{L:total_exp}
 Let $X$ be an $\F$-integrable \rv with values in $\Re^d$. For any sub-\sigalg~$\G \subseteq \F$, $\E \Big[\E \big[X \big| \G \big]\Big] = \E \big[X \big]$.
\end{proposition}

\begin{proposition}\label{L:factor_cond_exp}
 Let $X$ be an $\F$-integrable \rvrv and $\G \subseteq \F$ a sub-\sigalg. For any $\G$-integrable \rvrv $Y$ \st $YX$ is also integrable we have $\E \big[YX \big| \G \big] = Y \E \big[X \big| \G \big]$.
\end{proposition}

\begin{definition}[Stopping time]\label{D:stop_time}
 A random variable $\tau: \Omega \rightarrow \Na \cup \{+\infty\}$ is \emph{called stopping time} \wrt a filtration $(\F_n)_{n \in \Na}$ if for all $n \in \Na$, $\{ \tau =n \} \in \F_n$.
\end{definition}

\begin{definition}[\sigalg at stopping time]\label{D:stop_time_algebras}
 Let $\tau$ be a stopping time. An event \emph{prior to $\tau$} is any event $A \in \F_\infty$ \st $A \cap \{ \tau = n \} \in \F_n$ for all $n \in \Na$. The set of events prior to $\tau$ is a \sigalg denoted $\F_{\tau}$ and called \emph{\sigalg at time $\tau$}:
 \[ \F_{\tau} := \left\{A \in \F_{\infty}: ~ \forall n \in \Na,~A \cap \{\tau = n\} \in \F_n \right\}  \]
\end{definition}

\begin{proposition}\label{L:stop_time_algebras}
 Let $\tau_1$ and $\tau_2$ be two stopping times \wrt the same filtration $(\F_n)_{n \in \Na}$ \st $\tau_1 \leq \tau_2$ \as Then $\F_{\tau_1} \subseteq \F_{\tau_2}$.
\end{proposition}

\begin{definition}[Stopped Process]\label{D:process_stopped}
 Let $(X_n, \F_n)_{n \in \Na}$ be an adapted sequence with values in $\mathbbm{R}^d$. If $\tau$ is a stopping time \wrt the filtration $(\F_n)_{n \in \Na}$, then the \emph{process stopped at time $\tau$} (denoted by $X_\tau$) is the \rv defined as: \[\forall \omega \in \Omega, ~~ X_\tau(\omega) := \sum_{n \in \Na} X_n(\omega) \cdot \one{\tau(\omega) = n} ~~ \text{(\ie} ~ X_\infty(\omega) = 0 ~\text{by convention)} \]
\end{definition}

\begin{proposition}\label{L:process_stopped}
 $X_\tau$ --the process stopped at time $\tau$-- is $\F_{\tau}$-measurable.
\end{proposition}

%
\begin{definition}[Martingale difference sequence]\label{D:mds}
 An adapted sequence $(X_n, \F_n)_{n\in \Na}$ is a \emph{martingale difference sequence} (MDS for short) if for all $n \in \Na$, $X_n$ is $\F_n$-integrable and $\E\big[ X_{n+1} | \F_{n} \big] = 0$ \as
\end{definition}
%
%

\begin{proposition}[Azuma's inequality]
 Let $(X_n, \F_n)_{n\in \Na}$ be an MDS such that $|X_n| \leq a$ \as for all $n \in \Na$. Then for all $\delta \in ]0,1[$, 
 \begin{align*}
 \Prob \left( \forall n \geq 1,~ \left| \sum_{i =1}^{n} X_i \right|  \leq a \sqrt{n \ln \left( \frac{2n}{\delta} \right)} \right) \geq 1 - \delta
 \end{align*}
\end{proposition}
\begin{proof}
 Azuma's inequality states that:
 \begin{align*}
 \Prob \left( \left| \sum_{i =1}^{n} X_i \right|  \leq a \sqrt{\frac{n}{2}\ln \left( \frac{2}{\delta} \right)} \right) \geq 1 - \delta
 \end{align*}
 We can then choose $\delta \leftarrow \frac{\delta}{2n^2}$ and take a union bound over all possible values of $n \geq 1$. The result follows by noting that $\sum_{n\geq 1}  \frac{\delta}{2n^2} < \delta$.
\end{proof}

\begin{proposition}[Freedman's inequality]
 Let $(X_n, \F_n)_{n\in \Na}$ be an MDS such that $|X_n| \leq a$ \as for all $n \in \Na$. Then for all $\delta \in ]0,1[$, 
 \begin{align*}
 \Prob \left( \forall n \geq 1,~ \left| \sum_{i =1}^{n} X_i \right|  \leq 2 \sqrt{ \left( \sum_{i =1}^{n} \Varcond{X_i}{\F_{i-1}} \right) \cdot\lnn{ \frac{4n}{\delta} }  } + 4 a \lnn{\frac{4n}{\delta}} \right) \geq 1 - \delta
 \end{align*}
\end{proposition}
\begin{proof}
 \citet{freedman1975} showed that when $a =1$:
 \begin{align*}
 \Prob \left( \forall n \geq 1,~ \sum_{i =1}^{n} X_i \geq \varepsilon, ~ \sum_{i =1}^{n} \Varcond{X_i}{\F_{i-1}} \leq k\right) \leq \expp{ \frac{-\varepsilon^2}{2 k + 2\varepsilon / 3} }
 \end{align*}
 Since $(-X_n, \F_n)_{n\in \Na}$ is also an MDS, the above inequality holds also in absolute value (with a factor $2$ appearing in front of the exponential term after taking a union bound). In order to reverse the inequality (\ie replace $\varepsilon$ by $\delta$), we can use the same technique as \citet[Section 2]{Cesa-Bianchi:2005}. Finally, to account for the case where $a \neq 1$ we can simply apply the result to $(X_n/a, \F_n)_{n\in \Na}$.
\end{proof}

\subsection{A useful concentration with optional skipping}

In this section we prove a very simple theorem inspired by Doob's optional skipping \citep[\eg][Sec. 5.3, Lem. 4]{chow1988probability}. We start with useful definitions and lemmas.
%
\begin{lemma}\label{L:stop_time_algebras2}
 Let $\tau_1$ and $\tau_2$ be two stopping times \wrt the same filtration $(\F_n)_{n \in \Na}$. We say that $\tau_1 < \tau_2$ \as if $\Prob \left( \{ \tau_1 < \tau_2 \} \cup \{\tau_1 = \tau_2 = +\infty \} \right) = 1$. If $\tau_1 < \tau_2$ \as then $\F_{\tau_1 + 1} \subseteq \F_{\tau_2}$.
\end{lemma}
\begin{proof}
 If $\tau_1 < \tau_2$ then $\tau_1 +1 \leq \tau_2$ since $\tau_1$ is an integer-valued \rv If $\tau_1 = \tau_2 = +\infty$ then $\tau_1 +1 = +\infty$ and so $\tau_1 +1 = \tau_2$. In conclusion, $\tau_1 +1 \leq \tau_2$ \as and so by Prop.~\ref{L:stop_time_algebras}, $\F_{\tau_1 + 1} \subseteq \F_{\tau_2}$.
\end{proof}

\begin{definition}
 We say that a sequence of stopping times $(\tau_m)_{m \in \Na}$ \wrt $(\F_n)_{n\in \Na}$ is \emph{strictly increasing} if $\tau_m < \tau_{m+1}$ \as for all $m \geq 0$.
\end{definition}

\begin{lemma}\label{L:mds_stop_time}
 Let $(X_n, \F_n)_{n\in \Na}$ be a bounded adapted sequence and let $(\tau_m)_{m \in \Na}$ be a strictly increasing sequence of stopping times \wrt $(\F_n)_{n\in \Na}$.
For all $m \in \Na$, define $Y_{m} := X_{\tau_m+1} - \E \big[X_{\tau_m+1} \big| \F_{\tau_{m}} \big] $ and $\G_{m} :=\F_{\tau_{m \smallplus1}} $. Then, $(Y_m,\G_m)_{m\in\Na}$ is an MDS.
\end{lemma}
\begin{proof}
 By assumption, for any $m \in \Na$, $\tau_m < \tau_{m+1}$ \as and Prop.~\ref{L:stop_time_algebras} implies that $\F_{\tau_m} \subseteq \F_{\tau_{m\smallplus 1}}$. As a consequence, $(\G_m)_{m\in \Na} = (\F_{\tau_{m\smallplus 1}})_{m\in \Na}$ is a filtration. By Prop.~\ref{L:process_stopped} we know that $X_{\tau_m +1}$ is $\F_{\tau_m +1}$-measurable and Lem.~\ref{L:stop_time_algebras2} implies that $\F_{\tau_m + 1} \subseteq \F_{\tau_{m+1}} = \G_m$ so $X_{\tau_m +1}$ is $\G_m$-measurable.
 Finally, $\E \big[X_{\tau_m+1} \big| \F_{\tau_{m}} \big]$ is $\F_{\tau_m}$-measurable by definition (see Def.~\ref{D:cond_exp}). Therefore, $Y_m$ is $\G_m$-measurable.
 
 Since by assumption $X_n$ is \as bounded ($\Prob \left( X_n < K \right) =1$ for all $n\geq 0$), we can write \as (see Def.~\ref{D:process_stopped}) \[\big|X_{\tau_m +1}\big| = \left|\sum_{n=0}^{+\infty} \one{\tau_m +1 = n} \cdot X_n \right| \leq \sum_{n=0}^{+\infty} \one{\tau_m +1 = n} \cdot \left|X_n \right| \leq K \sum_{n=0}^{+\infty} \one{\tau_m +1 = n} = K\]
 Thus, $X_{\tau_m +1}$ is \as bounded hence integrable implying that $\E \big[X_{\tau_m+1} \big| \F_{\tau_m} \big]$ is well-defined (see Def.~\ref{D:cond_exp}).
 Therefore, $Y_{m}$ is \as bounded and so integrable.
 
 Finally, we can apply Prop.~\ref{L:factor_cond_exp} and we obtain:
 \begin{align*}
  \E \big[Y_{m+1} \big| \G_m \big] &= \E \Big[X_{\tau_m+1}  - \E \big[X_{\tau_m+1} \big| \F_{\tau_m} \big] \Big| \F_{\tau_m} \Big]
  = \E \big[X_{\tau_m+1} \big| \F_{\tau_m} \big] - \E \big[X_{\tau_m+1} \big| \F_{\tau_m} \big] = 0
 \end{align*} which concludes the proof.
\end{proof}

\begin{theorem}\label{T:inequality_stop_mds}
Let $(X_n, \F_n)_{n\in \Na}$ be an adapted sequence \as bounded by $a_1$ and let $(\tau_m)_{m \in \Na}$ be a strictly increasing sequence of stopping times \wrt $(\F_n)_{n\in \Na}$.
If $(Y_m,\G_m)_{m\in\Na}$ is defined as in Lem.~\ref{L:mds_stop_time} then the following concentration inequalities hold:
 \begin{align*}
  &\Prob \left( \forall m \geq 1, ~ \left| \sum_{i =1}^{m} Y_m  \right|  \leq a_1 \sqrt{m \ln \left( \frac{2 m }{\delta} \right)} \right) \geq 1 - \delta\\
 &\Prob \left( \forall m \geq 1,~ \left| \sum_{i =1}^{m} Y_m \right|  \leq 2 \sqrt{ \left( \sum_{i =1}^{m} \Varcond{Y_i}{\G_{i-1}} \right) \cdot\lnn{ \frac{4m}{\delta} }  } + 4 a_1\lnn{\frac{4m}{\delta}} \right) \geq 1 - \delta
 \end{align*}
 In particular for any $\F$-measurable integer-valued \rv $N: \Omega \rightarrow \Na$ the above inequalities hold true with $m$ replaced by $N$ \eg
  \begin{align*}
  \Prob \left(\left| \sum_{i =1}^{N} Y_m \right|  \leq a_1  \sqrt{N \ln \left( \frac{2 N }{\delta} \right)} \right) \geq 1 - \delta ~~ \dots
 \end{align*}
 
\end{theorem}
\begin{proof}
 The concentration inequalities follow from Lem.~\ref{L:mds_stop_time} and Azuma's and Freedman's inequalities. 
If the results hold for all $n \in \Na$ and $N$ takes values in $\Na$, then the result holds for $N$ too which concludes the proof.
\end{proof}
%

\subsection{In the regret proof}

For any $t \geq 0$, the \sigalg induced by the past history of state-action pairs and rewards up to time $t$ is denoted $\mathcal{F}_t:= \sigma \left( s_1, a_1, r_1, \dots, s_t, a_t \right)$ where by convention $\mathcal{F}_0 = \sigma \left( \emptyset \right)$ and $\mathcal{F}_\infty := \cup_{t\geq 0}\mathcal{F}_{t}$. 
Trivially, for all $t\geq 0$, $\mathcal{F}_{t} \subseteq \mathcal{F}_{t+1}$ and the filtration $\left(\mathcal{F}_t \right)_{t \geq 0}$ is denoted by $\mathbb{F}$. We recall that the sequence $(t_k)_{k \geq 1}$ (starting times of episodes $k\geq 1$) is formally defined by $t_1 :=1$ and for all $k \geq 1$, \[t_{k+1}:=1 + \inf\left\{T \geq t > t_k:~  \sum_{u=t_k}^{t -1 } \mathbbm{1}(s_u \in I(s),a_u =a) \geq \sum_{u=0}^{t_k -1 } \mathbbm{1}(s_u \in I(s) ,a_u =a)  \right\}.\]
where by convention $\inf \{ \emptyset \} := T $. It is immediate to see that for all $t\geq 0$, $\{ t_k = t \} \in \mathcal{F}_{t-1} \subseteq \mathcal{F}_t$ and so $t_k$ is a \emph{stopping time} \wrt filtration $\mathbb{F}$ (see Def.~\ref{D:stop_time}).


The following lemma is used in App.~\ref{app:bound_diff_tilde_bar}:
 
 \begin{lemma}\label{L:conditioning_stop_time}
 For all $l \geq 1$, we have:
 \begin{enumerate}[topsep=0pt, partopsep=0pt]
 \itemsep0em 
  \item $\E\big[ w^*(s_{\tau_l+1}) \big|\mathcal{G}_{l-1} \big] = \int_{\calS} p(s'|s_{\tau_l},a_{\tau_l})w^*(s')  \mathrm{d}s' $,
  \item $\E\big[ \one{s_{\tau_l+1}\in J} \big|\mathcal{G}_{l-1} \big] = \int_{J} p(s'|s_{\tau_l},a_{\tau_l}) \mathrm{d}s' $,
  \item and  $\E\big[ r_{\tau_l}(s_{\tau_l},a_{\tau_l}) \big|\mathcal{G}_{l-1} \big] = r(s_{\tau_l},a_{\tau_l})$.
 \end{enumerate}
\end{lemma}
\begin{proof} To prove this result, we rely on the definition of conditional expectation (see Def.~\ref{D:cond_exp}). \\
1) By Prop.~\ref{L:process_stopped}, $(s_{\tau_l},a_{\tau_l})$ is $\G_{l-1}$-measurable ($\G_{l-1} = \F_{\tau_l}$) and so $\int_{\calS} p(s'|s_{\tau_l},a_{\tau_l})w^*(s')  \mathrm{d}s'$ is $\G_{l-1}$-measurable too. Moreover, $\left|\int_{\calS} p(s'|s_{\tau_l},a_{\tau_l})w^*(s')  \mathrm{d}s' \right| \leq c/2$ \as so $\int_{\calS} p(s'|s_{\tau_l},a_{\tau_l})w^*(s')  \mathrm{d}s'$ is also integrable (and therefore $\G_{l-1}$-integrable).\\
2) We recall that for any stochastic process $(X_t)_{t\geq 0}$, we use the convention that $X_{\infty} =0$ \as implying that $X_{\tau_l} = \sum_{t=0}^{+\infty} X_t \one{\tau_l = t }$ (see Def.~\ref{D:process_stopped}). Usinng the law of total expectations (see Prop.~\ref{L:total_exp}) we have that $\forall A \in \mathcal{G}_{l-1}$,
\begin{align*}
 \E \big[ \mathbbm{1}(A) \times w^*(s_{\tau_l+1}) \big] 
 &= \sum_{t=0}^{+\infty} \E \big[  \mathbbm{1}(A \cap \{\tau_l = t \}) \times w^*(s_{\tau_l+1}) \big] \\
 &= \sum_{t=0}^{+\infty} \E \bigg[  \E \Big[ \mathbbm{1}(\underbrace{A \cap \{\tau_l = t \}}_{\in \mathcal{F}_t}) \times w^*(\underbrace{s_{\tau_l+1}}_{=s_{t+1}}) \Big| \mathcal{F}_t \Big] \bigg]
\end{align*}
In the first equality, the fact that we can move the sum outside the expectation is a direct consequence of the dominated convergence theorem (for series) since 
\begin{align*}
 \sum_{t=0}^{+\infty} \E \big[  \mathbbm{1}(A \cap \{\tau_l = t \}) \times \big|w^*(s_{\tau_l+1}) \big| \big] &\leq c/2 \sum_{t=0}^{+\infty} \E \big[  \mathbbm{1}(A \cap \{\tau_l = t \}) \big] \\&= c/2\sum_{t=0}^{+\infty} \Prob \left( A \cap \{\tau_l = t \} \right) = c/2 \cdot  \Prob (A) < +\infty
\end{align*}
Under event $\{\tau_l = t \}$ we have that $s_{\tau_l+1} = s_{t+1}$ \as Moreover, $A \cap \{\tau_l = t \} \in \F_t$ since $\tau_l$ is a stopping time (see Def.~\ref{D:stop_time_algebras}) so by Prop.~\ref{L:factor_cond_exp} we can move it outside the conditional expectation and we get:
\begin{align*}
 \E \big[ \mathbbm{1}(A) \times w^*(s_{\tau_l+1}) \big] 
 &= \sum_{t=0}^{+\infty} \E \bigg[ \mathbbm{1}(A \cap \{\tau_l = t \}) \times \underbrace{\E \Big[   w^*(s_{t+1}) \Big| \mathcal{F}_t \Big]}_{= \int_{\calS} p(s'|s_t,a_t)w^*(s')  \mathrm{d}s'} \bigg]\\
  &= \E \bigg[\mathbbm{1}(A) \times \underbrace{\sum_{t=0}^{+\infty}\mathbbm{1}(\tau_l = t ) \int_{\calS} p(s'|s_t,a_t)w^*(s')  \mathrm{d}s'   }_{=\int_{\calS} p(s'|s_{\tau_l},a_{\tau_l})w^*(s')  \mathrm{d}s' \text{  (see Def.~\ref{D:process_stopped})}}\bigg]\\
 &= \E \left[\mathbbm{1}(A) \times \int_{\calS} p(s'|s_{\tau_l},a_{\tau_l})w^*(s')  \mathrm{d}s' \right]
\end{align*}
This proves the first inequality (see Def.~\ref{D:cond_exp}). The second and third equality can be proved using the same technique.
\end{proof}